\documentclass{article}

     \usepackage[preprint,nonatbib]{neurips_2025}

\usepackage[utf8]{inputenc} 
\usepackage[T1]{fontenc}    
\usepackage{hyperref}       
\usepackage{url}            
\usepackage{booktabs}       
\usepackage{amsfonts}       
\usepackage{nicefrac}       
\usepackage{microtype}      
\usepackage{xcolor}         

\usepackage{enumerate}
\usepackage{graphicx}
\usepackage{amssymb}
\usepackage{mathtools}
\usepackage{subfig}
\usepackage{xspace}
\usepackage[lined,linesnumbered,ruled,commentsnumbered]{algorithm2e} 
\usepackage{colonequals}
\usepackage[mathscr]{euscript}
\usepackage{fixltx2e}    
\usepackage{amsthm}
\usepackage{mathrsfs}
\usepackage{caption}
\usepackage{float}
\usepackage{relsize}
\usepackage{fnpos} 
\usepackage{wrapfig}

\usepackage{tikz}
\usepackage{pgfplots}
\usetikzlibrary{arrows,shapes,backgrounds,plotmarks,positioning}
\pgfplotsset{compat=newest}                         
\pgfplotsset{plot coordinates/math parser=false}
\newlength\figureheight
\newlength\figurewidth

\newtheorem{theorem}{Theorem}

\newtheorem{remark}{Remark}

\newcommand{\E}{\mathbb{E}}
\newcommand{\X}{\mathcal{X}}
\newcommand{\Y}{\mathcal{Y}}
\newcommand{\Z}{\mathcal{Z}}

\newcommand{\Set}[1]{\{#1\}}

\newcommand*\dif{\mathop{}\!\mathrm{d}}

\title{$\alpha$-GAN by R\'{e}nyi Cross Entropy}

\author{%
  Ni~Ding
  \\
  University of Auckland\\
  New Zealand \\
  \texttt{ni.ding@auckland.ac.nz} \\
   \And
   Miao Qiao \\
   University of Auckland \\
   New Zealand \\
   \texttt{miao.qiao@auckland.ac.nz} \\
   \AND
   Jiaxing Xu \\
   Nanyang Technological University \\
   Singapore \\
   \texttt{jiaxing003@e.ntu.edu.sg} \\
   \And
   Yiping Ke \\
   Nanyang Technological University \\
   Singapore \\
   \texttt{ypke@ntu.edu.sg} \\
   \And
   Xiaoyu Zhang \\
   University of Auckland \\
   New Zealand \\
   \texttt{xsha710@aucklanduni.ac.nz} \\
}

\begin{document}

\maketitle

\begin{abstract}	
	This paper proposes $\alpha$-GAN, a generative adversarial network using R\'{e}nyi measures. The value function is formulated, by R\'{e}nyi cross entropy, as an expected certainty measure incurred by the discriminator's soft decision as to where the sample is from, true population or the generator. 
	The discriminator tries to maximize the R\'{e}nyi certainty about sample source, while the generator wants to reduce it by injecting fake samples. This forms a min-max problem with the solution parameterized by the R\'{e}nyi order $\alpha$.   
	This $\alpha$-GAN reduces to vanilla GAN at $\alpha = 1$, where the value function is exactly the binary cross entropy. 
	The optimization of $\alpha$-GAN is over probability (vector) space. 
	It is shown that the gradient is exponentially enlarged when R\'{e}nyi order is in the range $\alpha \in (0,1)$. This makes convergence faster, which is verified by experimental results. A discussion shows that choosing $\alpha \in (0,1)$ may be able to solve some common problems, e.g., vanishing gradient. A following observation reveals that this range has not been fully explored in the existing R\'{e}nyi version GANs. 
\end{abstract}

\section{Introduction}

Generative model produces new data points that simulate the true distribution. It is a powerful tool in learning underlying data structure and has been widely used nowadays in many areas, e.g., healthcare, data augmentation, natural language processing. 
Generative adversarial network (GAN)~\cite{Goodfellow2014_VanillaGAN} is one type of generative models that involves two parties, the generator and the discriminator. The goal is to train the generator to learn the probability distribution of a population, while the discriminator is there to help the generator to reach this goal. However, the optimization is designed in a competitive manner: 
the discriminator tries to minimize the estimation loss when determining whether the sample is generated from the actual population or the generator, while the generator tries to fabricate samples good enough to fool the discriminator to deteriorate its performance. 
This naturally forms an adversarial network (more precisely, a competitive zero-sum game), where the training is done by an alternating maximization and minimization process. It is known that GAN is good at generating complex, high-resolution but realistic images~\cite{Zhang2022CVPR,Liao2022CVPR}.

When GAN was proposed~\cite{Goodfellow2014_VanillaGAN}, the value function is formulated as a superposition of competing objectives between the discriminator and generator. 
This formulation implies a connection to cross entropy, but is more explicitly interpreted as a Jensen-Shannon divergence measuring the dissimilarity between true data probability and generator's distribution. The overall framework uses the expected log-loss, as for Shannon entropy.  
Following the divergence interpretation, there have been least square GAN, $\chi^2$-GAN~\cite{Tao2018}, etc., proposed in the literature. All of them can be generalized by $f$-GAN~\cite{Nowozin2016_fGAN} using the variational representation of $f$-divergence~\cite{Nguyen2010_FDiv}. 
Recently, there are various attempts, e.g.,~\cite{Bhatia2021_KthOrder_Renyi_GAN,Sarraf2021_RGAN,Welfert2024,Veiner2024AlphaGAN,Kurri2022ISIT_AlphaGAN,Kurri2021ITW_AlphaGAN}, trying to generalize vanilla GAN~\cite{Goodfellow2014_VanillaGAN} to $\alpha$-GAN, based on the fact that R\'{e}nyi entropy generalizes Shannon entropy but provides various levels of uncertainty measures via the R\'{e}nyi order $\alpha$. 
Instead of changing $f$ function in $f$-GAN which usually involves rebuilding the model, this generalization introduces a hyperparameter $\alpha$ that is adjustable by the user. 

It is discovered in~\cite{Sarraf2021_RGAN,Welfert2024,Veiner2024AlphaGAN} that tuning $\alpha$ helps deal with some common problems in GAN training, e.g., model collapse, vanishing gradient. 
This is consistent with the observations in other machine learning studies. For example, the Tsallis entropy is introduced in~\cite{LeeS2018} to improve the imitation learning; it is shown in~\cite{Zhang2018} that a Box-Cox transformation tunable between mean absolute error and cross entropy enhances robustness against noisy labels for deep neural networks. 
However, the difficulty is that there does not exist a well-defined R\'{e}nyi cross entropy. For example, the minimum of the R\'{e}nyi cross entropy in~\cite{Bhatia2021_KthOrder_Renyi_GAN} is not R\'{e}nyi entropy; the $\alpha$-loss adopted in~\cite{Welfert2024,Veiner2024AlphaGAN,Kurri2022ISIT_AlphaGAN,Kurri2021ITW_AlphaGAN} is the Arimoto divergence that is not defined at $\alpha = 0$ (see Section~\ref{sec:discuss} for the explanation). 
In addition, the purpose of $\alpha$ has not been fully studied. 
There are some analysis on $\alpha$ in \cite{Kurri2022ISIT_AlphaGAN, Sarraf2021_RGAN}. But, none of them explains exactly how this R\'{e}nyi order can help the model training.  

In this paper, we propose an $\alpha$-GAN and formulate the value function using the R\'{e}nyi cross entropy in~\cite{Ding2024_ISIT}. It measures the certainty incurred by the discriminator's soft decision wrt the randomness created jointly by the true distribution and generator's probability. 
We show that the value function is a quasi-arithmetic mean that is well defined in the R\'{e}nyi order range $\alpha \in [0,\infty)$. It reduces to binary cross entropy at $\alpha = 1$, where $\alpha$-GAN becomes vanilla GAN. 
The discriminator maximizes the value function (i.e., the estimation certainty) and the maximum equals to the negative of the R\'{e}nyi entropy. In the next step, the generator minimizes the value or discriminator's certainty. It is equivalent to the maximization of R\'{e}nyi entropy and the answer is uniform distribution, where the generator reproduces true distribution. 
For the alternating gradient ascent and descent steps in the training, we show that the gradient magnitude is exponentially decreasing in $\alpha$. As $\alpha$ deceases towards $0$, the gradient is enlarged greatly. It suggests an accelerated gradient approach for solving the vanishing gradient problem and improving the learning speed by choosing $\alpha \in (0,1)$. 
This is verified by experimental results on 1D Gaussian and MNIST, where the $\alpha$-GAN converges faster at $\alpha = 0.05$ and $0.1$. 
We finally review and compare the existing works on $\alpha$-GAN in~\cite{Welfert2024,Veiner2024AlphaGAN,Kurri2022ISIT_AlphaGAN,Kurri2021ITW_AlphaGAN}, where we point out that the range $\alpha\in(0,1)$ has not be fully explored previously.

\paragraph{Notation} 

Capital and lowercase letters denote random variable (r.v.) and its elementary event or instance, respectively, e.g., $x$ is an instance of $X$. Calligraphic letters denote sets, e.g., $\X$ refers to the alphabet of $X$. We assume finite countable alphabet.
Let $P_{X}(x)$ be the probability of outcome $X = x$ and $P_X  = (P_{X}(x) \colon x \in \X)$ be the probability mass function.
The expected value of $f(X)$ w.r.t. probability $P_{X}$ is denoted by $\E_{X \sim P_X}[f(X)] = \sum_{x \in \X} P_{X}(x) f(x)$.
For the conditional probability $P_{Y|X} = (P_{Y|X}(y|x) \colon x\in\X, y\in\Y)$, $P_{Y|X=x} = (P_{Y|X}(y|x) \colon y\in\Y)$ refers to the probability distribution of $Y$ given the event $X = x$.
%

\paragraph{Organization} Section~\ref{sec:sys} describes the system model of the proposed $\alpha$-GAN, emphasizing on the $\alpha$-parameterized value function and its connection to R\'{e}nyi cross entropy. Section~\ref{sec:solution} derives the solution of the min-max problem, where we reveal the R\'{e}nyi order $\alpha$ can be properly chosen to accelerate gradient. Section~\ref{sec:Exp} shows experiments results on 1D Gaussian and MNIST. Section~\ref{sec:relate} reviews related works and points out the underexplored region $\alpha \in (0,1)$. Section~\ref{sec:discuss} discusses the possible extensions of this paper, followed by the conclusion. 

\section{System Model}
\label{sec:sys}

%
Let $P_r$ be the true distribution of a population $X$. At the same time, we have a generator $G$ emits samples from distribution $P_g$. The objective is to have a generator $G$ emits a probability distribution $P_g$ equivalent to or approaching $P_r$. Assume that samples are drawn alternatively from the true population and generator. 
We introduce a discriminator $D$ as a helper to train the generator. 

For each sample $x$, let $Z$ be a random variable indicates two possible outcomes $\Z = \Set{\text{`r'}, \text{`g'}}$ as to the sample source: `r' denotes that sample $x$ is generated from true distribution; `g' denotes sample $x$ is emitted from the generator.
The discriminator makes a soft decision $\hat{P}_{Z|X=x}$ in terms of the probability distribution of $Z$. 
Denote $D(x) = \hat{P}_{Z|X}(\text{`r'} | x)$ and $1-D(x) = \hat{P}_{Z|X}(\text{`g'} | x)$ the outputs of the discriminator.  
The actual distribution $P_{Z|X}$ is defined as the fractional chances as 
\begin{equation*}
	\begin{aligned}
		P_{Z|X}(\text{`r'} | x) & = \frac{P_r(x)}{P_r(x) + P_g(x)}, \\
		P_{Z|X}(\text{`g'} | x) & = 1 - \frac{P_r(x)}{P_r(x) + P_g(x)}  = \frac{P_g(x)}{P_r(x) + P_g(x)}.
	\end{aligned}
\end{equation*}

\subsection{Problem Formulation}

There are two optimization steps involved in a GAN, one is over the discriminator $D$ and the other is over $P_g$ for training the generator. 
Let there be $N$ sample instances $x$. We define the \emph{value function} of $\alpha$-GAN by 
\begin{align} 
	V_\alpha (D, P_g)
	& = \frac{\alpha}{\alpha-1} \log \frac{1}{N}\sum_{x} \Big( P_{Z|X}(\text{`r'}|x) \hat{P}_{Z|X}^{\frac{\alpha-1}{\alpha}} (\text{`r'}|x)   +   P_{Z|X}(\text{`g'}|x) \hat{P}_{Z|X}^{\frac{\alpha-1}{\alpha}} (\text{`g'}|x) \Big) \label{eq:V} \\
	& =  \frac{\alpha}{\alpha-1} \log \frac{1}{N} \sum_{x} \Big( \frac{P_r(x)}{P_r(x) + P_g(x)} D(x)^{\frac{\alpha-1}{\alpha}}  +  \frac{P_g(x)}{P_r(x) + P_g(x)} (1-D(x))^{\frac{\alpha-1}{\alpha}} \Big) \nonumber
\end{align}
for all $\alpha \in [0,+\infty)$. 
The definition is closely related to the a R\'{e}nyi cross entropy proposed in \cite{Ding2024_ISIT} based on the original intuition to define R\'{e}nyi measures as a generalized $f$-mean~\cite{Renyi1961_Measures}.

\paragraph{Certainty measure interpretation by R\'{e}nyi cross entropy equivalence} 
Let $P_{XZ}(x,z) = P_{Z|X}(z|x) P_X(x) $ for all $x$ and $z$, where $P_{X}(x)$ is the empirical distribution from the sample set. We can rewrite~\eqref{eq:V} as 
\begin{align} \label{eq:Obj}
	V_\alpha (D, P_g)
	& = \frac{\alpha}{\alpha-1} \log \E_{XZ \sim P_{XZ}} \Big[  \hat{P}_{Z|X}^{\frac{\alpha-1}{\alpha}}(Z|X)  \Big]   \nonumber \\
	& = - H_\alpha (P_{Z|X}, \hat{P}_{Z|X})
\end{align}
where 
$	H_\alpha (P_{Z|X}, \hat{P}_{Z|X})
	= \frac{\alpha}{1-\alpha} \log \E_{XZ \sim P_{XZ}} \big[ ( \hat{P}_{Z|X}^{-1}(Z|X) )^{\frac{1-\alpha}{\alpha}}  \big]
$
is the R\'{e}nyi conditional cross entropy proposed in~\cite{Ding2024_ISIT} that measures the expected uncertainty incurred by the soft decision incurs $\hat{P}_{Z|X}$ in the form of generalized $f$-mean for $f (t)= \exp(\frac{1-\alpha}{\alpha}t) $. 
Alternatively, writing 
\begin{equation}\label{eq:powermean}
	V_\alpha (D, P_g) = \log ( \E_{XZ \sim P_{XZ}} [  \hat{P}_{Z|X}^{\frac{\alpha-1}{\alpha}} (Z|X) ]  )^{\frac{\alpha}{\alpha-1} } ,
\end{equation}
we have the value function being the logarithm of an $\frac{\alpha-1}{\alpha}$-order weighted power mean, quantifying the average gain incurred by the decision variable $\hat{P}_{Z|X}$. 
The exponential order $\frac{\alpha-1}{\alpha}$ (or $\frac{1-\alpha}{\alpha}$) is due to the R\'{e}nyi order shift in optimization, i.e., the optimizer coincides with the R\'{e}nyi entropy (see proof of Theorem~\ref{theo:Main}). 
In this sense, $V_\alpha(D, P_g)$ is measuring the certainty incurred by $\hat{P}_{Z|X}$ given the true conditional probability $P_{Z|X}$ as an $f$-mean for $f(t) = \exp(\frac{\alpha-1}{\alpha}t) $.

At order $\alpha= 1$, we have the limit of $V_\alpha (D, P_g)$ being (by L'H\^{o}pital's rule)
\begin{align}
	V_1 (D, P_g)  &=   - H_1 (P_{Z|X}, \hat{P}_{Z|X})  \label{eq:BCE} \\
	&= \frac{1}{N} \sum_{x} \Big( \frac{P_r(x)}{P_r(x) + P_g(x)} \log D(x) +  \frac{P_g(x)}{P_r(x) + P_g(x)} \log (1-D(x))  \Big), 
\end{align}
Here, $ H_1 (P_{Z|X}, \hat{P}_{Z|X})$ is exactly the Shannon order cross entropy that is widely used in machine learning. We will show in Section~\ref{sec:solution} that $\alpha$-GAN reduces to vanilla GAN \cite{Goodfellow2014_VanillaGAN}, while \eqref{eq:BCE} shows more straightforwardly connection to binary cross entropy, the loss function usually used in GAN training.
At order $\alpha= \infty$, we have 
\begin{align}
	V_\infty(D,P_g) & = - H_\infty (P_{Z|X}, \hat{P}_{Z|X})  \nonumber  \\
	& = \log \frac{1}{N} \sum_{x} \Big(  \frac{P_r(x)}{P_r(x) + P_g(x)} D(x) +  \frac{P_g(x)}{P_r(x) + P_g(x)} (1-D(x))  \Big)
\end{align}
being the logarithm of expected $0$-$1$ gain function (taking negative of it gives exactly the $0$-$1$ loss). 
Order $\alpha = 0$ will generate an value function $V_0 (D, P_g)  = \log \inf_{x,z}   \hat{P}_{Z|X}(z|x) =  \log \min_{x}  \min \Set{D(x), 1-D(x)} $. This order is closely related to nonstochasic decision making. We will not elaborate it in this paper, as GAN training is stochastic. Instead, we will focus on the range $\alpha \in (0,1)$. 

Assume countable and finite sample space.
The problem of training the generator $G$ to produce a probability distribution $P_g$ that approximates the real one $P_r$ is formulated by an alternating optimization problem as
\begin{equation} \label{eq:MainObj}
	\min_{P_g} \max_{D}    V_\alpha (D, P_g).
\end{equation}
for all $\alpha \in [0,\infty)$. 
The interpretation of this min-max problem, using the cross entropy expression~\eqref{eq:Obj}, is that the discriminator tries to seek the optimal soft decision $D^*$ that maximizes the certainty when discriminating the sample source. 
This will bring $D^*$ closer to the actual probability $P_{Z|X}(\text{`r'} | x)$. 
On the other hand, the generator tries to minimize the certainty, as it wants to insert fake samples to fool the discriminator. This can be done by making uniform distribution $P_{Z|X}(\text{`r'} | x) = P_{Z|X}(\text{`g'} | x)$. In this case, we must have $P_g = P_r$ and the final goal of GAN training attains.

\section{Solution}
\label{sec:solution}

The generator and discriminator have conflicting goals and therefore form a competitive (zero-sum) game. We show below that the saddle point solution of min-max problem~\eqref{eq:MainObj} is parameterized by R\'{e}nyi order $\alpha$. The proof reveals a connection to a R\'{e}nyi conditional entropy measure.

\begin{theorem} \label{theo:Main}
	For all $\alpha \in [0,\infty)$,
	\begin{enumerate}[(a)]
		\item The maximizer of $\max_{D} V_\alpha(D, P_g)$ is
		\begin{equation} \label{eq:DStar}
			D^*(x) = \frac{P_r^{\alpha}(x)}{P_r^{\alpha}(x) + P_g^{\alpha}(x)}, \qquad \forall x;
		\end{equation}	
		\item $\min_{P_g}   V_\alpha (D^*, P_g)  = -\log 2$ with the minimizer being $P_g^* (x) = P_r (x), \forall x $.
	\end{enumerate}
\end{theorem}
\begin{proof}
	(a) can be proved by using the result in \cite{Ding2024_ISIT}. As $ \min_{P_g} \max_{D}    V_\alpha (D, P_g) = -  \max_{P_g} \min_{D}  H_\alpha (P_{Z|X}, \hat{P}_{Z|X}) $. For each $P_g$, i.e., fixed $P_{Z|X}$, $\min_{D}  H_\alpha (P_{Z|X}, \hat{P}_{Z|X})  = H_\alpha(P_{Z|X})$ \cite[Theorem~1]{Ding2024_ISIT} with the minimizer	
	$
		\hat{P}^*_{Z|X}(\text{`r'} |x) = \frac{P_{Z|X}^{\alpha}(\text{`r'}|x)}{P_{Z|X}^{\alpha}(\text{`r'}|x) + P_{Z|X}^{\alpha}(\text{`g'}|x)}.
	$	
	This gives~\eqref{eq:DStar}. 
	Here, $H_{\alpha} (P_{X|Y}) = \frac{\alpha}{1-\alpha} \log \frac{1}{N}\sum_{x} [ ( P_{Z|X}^{\alpha}(\text{`r'}|x) + P_{Z|X}^{\alpha}(\text{`g'}|x) )^{\frac{1}{\alpha}} ] $ is the Arimoto conditional entropy~\cite{Arimoto1977}, 
	The next problem is
	\begin{equation}
		\min_{P_g}  V_\alpha (D^*,P_g) = - \max_{P_g} H_{\alpha} (P_{Z|X}),
	\end{equation}
	where the maximum entropy is incurred by uniform distribution $P_{Z|X}^* = \frac{1}{2}$ which gives the minimizer $P_g^* = P_r$. 	
\end{proof}

\begin{remark}[$1$-GAN $\Leftrightarrow$ vanilla GAN] 
For $\alpha=1$, the solution \eqref{eq:DStar} is $D^*(x) = \frac{P_r(x)}{P_r(x) + P_g(x)},\forall x$. This coincides with vanilla GAN~\cite{Goodfellow2014_VanillaGAN}, where the value function is 
\begin{equation}
	V(D,P_g) = \frac{1}{N} \sum_{x} \big( P_r(x) \log D(x) +  P_g(x) \log (1-D(x)) \big). 
\end{equation}
This differs $V_1(D,P_g)$ in a denominator $P_r(x) + P_g(x)$ for each sample $x$. As the optimization is separable in sample space, the denominator can be treated as a nonnegative scaler and therefore the optimizer remains the same. However, the definition $V_1(D,P_g)$ shows explicitly a relationship to the binary cross entropy. 
\end{remark}

\subsection{Optimal discriminator $D^*$}

In the proof of Theorem~\ref{theo:Main}, the best soft decision $\hat{P}^*_{Z|X}(\text{`r'} |x)$ is an $\alpha$-scaled probability of $P_{Z|X}(\text{`r'} |x)$ that assigns more probability mass to high chance outcome of $Z$. The resulting optimal discriminator  $D^*(x) = \frac{P_r^{\alpha}(x)}{P_r^{\alpha}(x)+P_g^{\alpha}(x)} = \frac{1}{1+(P_g(x)/P_r(x))^{\alpha}}$ is monotonic in $\alpha$ in two cases. For $P_{g}  < P_{r}$,   $D^*(x)$ is increasing in $\alpha$; For $P_{g}  > P_{r}$,   $D^*(x)$ is decreasing in $\alpha$. See Figure~\ref{fig:Analysis}(a). 
It is clear that as $\alpha$ increase to very large, the optimal soft decision becomes maximum likelihood (a hard decision). However, $D^*(x) $ remains $0.5$ always, if $\alpha=0$ or $P_r(x) = P_g(x)$.

\begin{figure}[t]
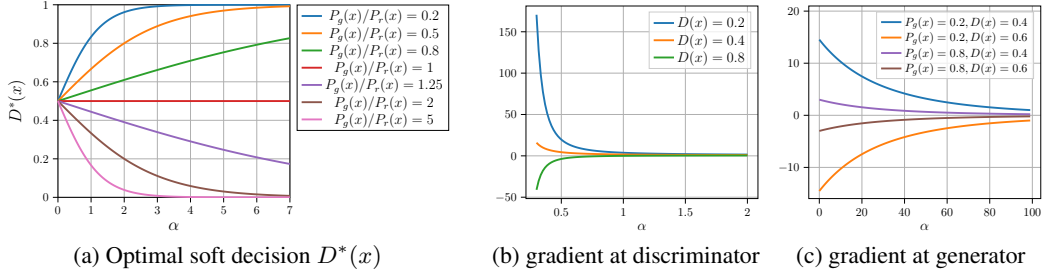

	\subfloat[Optimal soft decision $D^*(x)$]{
		\scalebox{0.45}{\input{figures/DStar_vs_Alpha.tex}}
	}\quad 
	\subfloat[gradient at discriminator]{
		\scalebox{0.45}{\input{figures/GradientD.tex}}
	}
	\subfloat[gradient at generator]{
		\scalebox{0.45}{\input{figures/GradientG.tex}}
	}
	\caption{(a) the optimal discriminator $D^*(x)$ in~\eqref{eq:DStar}; (b) the gradient when training the discriminator: RHS in~\eqref{eq:GradD};  the gradient when training the generator: RHS in~\eqref{eq:GradG}
	}
	\label{fig:Analysis}
\end{figure}

\subsection{Role of $\alpha$: accelerated gradient}
\label{sec:Alpha}

Let $\theta_D$ and $\theta_{P_g}$ be the parameters of discriminator and generator, respectively. The $\alpha$-GAN can be trained by alternating the maximization and minimization of $V_\alpha(D,P_g)$ over $D$ and $P_g$, respectively. This is accomplished by gradient ascent $\theta_{D} \coloneqq \theta_{D} + \nabla_{\theta_{D}} V_\alpha(D,P_g)$ followed by gradient descent  $\theta_{P_g} \coloneqq \theta_{P_g} - \nabla_{\theta_{P_g}} V_\alpha(D,P_g)$ iteratively. Both gradients are parameterized by the R\'{e}nyi order $\alpha$. 
The training algorithm can follow \cite[Algorithm~1]{Goodfellow2014_VanillaGAN}, except that the binary cross entropy should be replaced by $V_\alpha(D,P_g)$. 
We focus on how the two gradients vary with the R\'{e}nyi order $\alpha$.

For the discriminator, the partial derivative of $V_\alpha(D,P_g)$ wrt the soft decision $D(x)$ for each $x$ is 
\begin{equation}\label{eq:GradD}
	\frac{\partial V_\alpha(D,P_g)}{\partial D(x)} \propto \frac{P_r(x)}{P_r(x)+P_g(x)} D^{-\frac{1}{\alpha}}(x)  - \frac{P_g(x)}{P_r(x)+P_g(x)} (1-D(x))^{-\frac{1}{\alpha}}.
\end{equation}
The absolute value of RHS is exponentially decreasing in $\alpha$. See Figure~\ref{fig:Analysis}(b). The magnitude of this gradient is greatly enlarged as $\alpha \to 0$.
For the generator, the gradient wrt $P_g$ is 
\begin{equation}\label{eq:GradG}
	\frac{\partial V_\alpha(D,P_g)}{\partial P_g(x)} \propto \frac{\alpha}{\alpha-1} \frac{P_r(x)}{(P_r(x)+P_g(x))^2} \big( (1-D(x))^{\frac{\alpha-1}{\alpha}} - D^{\frac{\alpha-1}{\alpha}}(x) \big),
\end{equation}
which is also increasing as $\alpha$ approaches $0$. See Figure~\ref{fig:Analysis}(c).

The accelerated gradient can fasten learning speed. This is the main advantage of introducing R\'{e}nyi order $\alpha$ as a hyperparameter into GAN. Therefore, setting $\alpha \in (0,1)$ is anticipated to enhance the convergence rate, which will be experimentally verified in Section~\ref{sec:Exp}. 
This is in fact a method for solving the vanishing gradient problem~\cite{VanishGrad1998}, one of the common challenges in generative model training where the gradient diminishes when propagating backward in neural network. 
Usually, the problem is alleviated by introducing extra mechanisms, e.g., implementing an instantaneous gradient decay rate detector~\cite{Han2021}, a combination of smooth and nonsmooth functions (usually involves regularization)~\cite{Hu2009accelerated}, restarting the algorithm for a fixed number of iterations~\cite{Donoghue2013}. 
However, the proposed $\alpha$-GAN has accelerated gradient design nested in its loss function already, where the acceleration can be varied by changing the hyperparameter $\alpha$.

\begin{figure}[t]
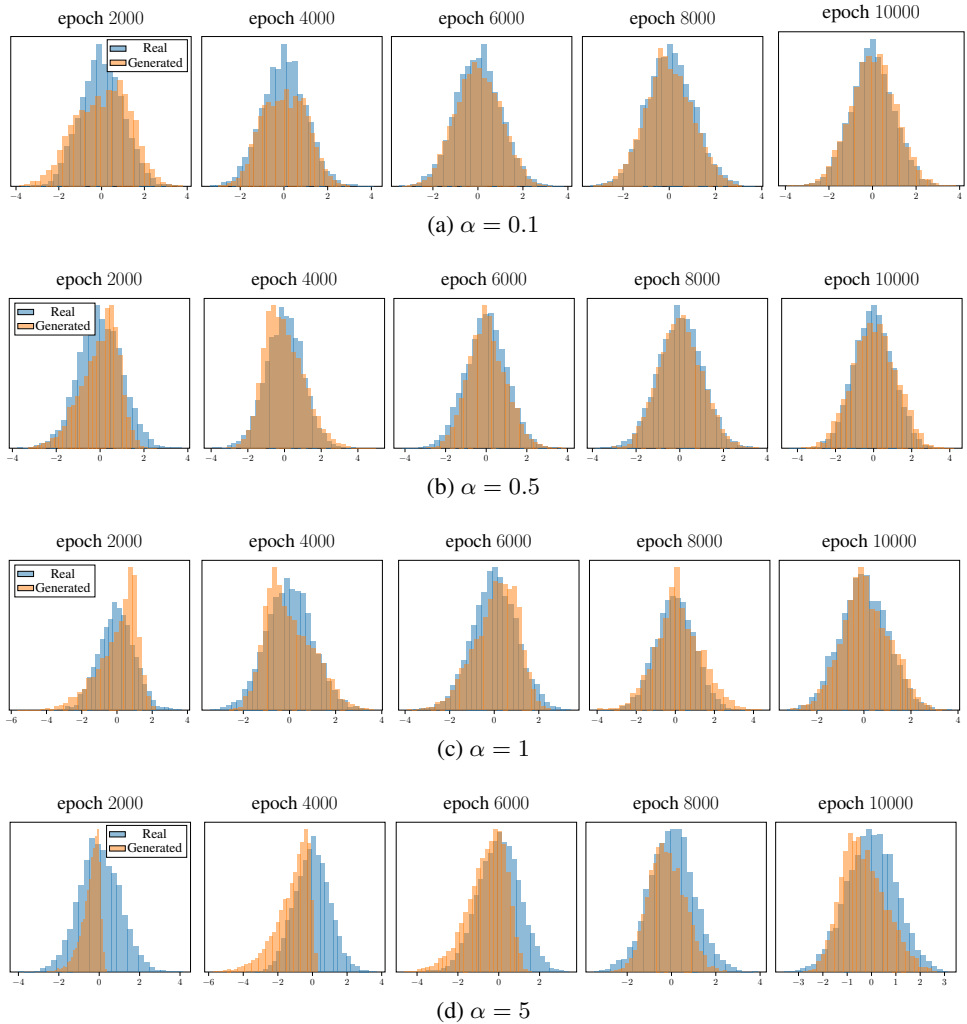

	\centering
	\subfloat[$\alpha = 0.1$]{
		\scalebox{0.35}{\input{figures/1DGaussHist_0.1_2000.tex}}
		\scalebox{0.35}{\input{figures/1DGaussHist_0.1_4000.tex}} 
		\scalebox{0.35}{\input{figures/1DGaussHist_0.1_6000.tex}}
		\scalebox{0.35}{\input{figures/1DGaussHist_0.1_8000.tex}} 
		\scalebox{0.36}{\input{figures/1DGaussHist_0.1_10000.tex}}}\\
	\subfloat[$\alpha = 0.5$]{
		\scalebox{0.35}{\input{figures/1DGaussHist_0.5_2000.tex}}
		\scalebox{0.35}{\input{figures/1DGaussHist_0.5_4000.tex}} 
		\scalebox{0.35}{\input{figures/1DGaussHist_0.5_6000.tex}}
		\scalebox{0.35}{\input{figures/1DGaussHist_0.5_8000.tex}} 
		\scalebox{0.35}{\input{figures/1DGaussHist_0.5_10000.tex}}	
	}\\
	\subfloat[$\alpha = 1$]{
		\scalebox{0.35}{\input{figures/1DGaussHist_1_2000.tex}}
		\scalebox{0.35}{\input{figures/1DGaussHist_1_4000.tex}} 
		\scalebox{0.35}{\input{figures/1DGaussHist_1_6000.tex}}
		\scalebox{0.35}{\input{figures/1DGaussHist_1_8000.tex}} 
		\scalebox{0.35}{\input{figures/1DGaussHist_1_10000.tex}}
	}\\
	\subfloat[$\alpha = 5$]{
		\scalebox{0.35}{\input{figures/1DGaussHist_5_2000.tex}}
		\scalebox{0.35}{\input{figures/1DGaussHist_5_4000.tex}} 
		\scalebox{0.35}{\input{figures/1DGaussHist_5_6000.tex}}
		\scalebox{0.35}{\input{figures/1DGaussHist_5_8000.tex}} 
		\scalebox{0.35}{\input{figures/1DGaussHist_5_10000.tex}}
	}
	\caption{Training $\alpha$-GAN for learning $1$D Gaussian for different values of $\alpha$, where $\alpha=1$ in (c) corresponds to the vanilla GAN~\cite{Goodfellow2014_VanillaGAN}. The figure shows real and generated probability distributions for every $2000$ epochs.}
	\label{fig:1DGaussian}
\end{figure}

\section{Experiment}
\label{sec:Exp}
We first run experiments for training the proposed $\alpha$-GAN for 1D Gaussian and then the MNIST dataset. The main purpose is to observe the convergence performance by varying the value of R\'{e}nyi order $\alpha$. 

\paragraph{1D-Gaussian}
We implement GAN as follows. Latent dimension is set to $5$. Learning rate is $0.0002$ and batch size is $128$. The discriminator and generator both have two layers, fully connected, using ReLu activation function. The discriminator applies sigmoid activation function to output a probability. 
We vary R\'{e}nyi order $\alpha \in \Set{0.1, 0.5, 1, 5}$ and run $10^4$ epochs for each value. The results are shown in~Figure~\ref{fig:1DGaussian}, where at $\alpha=0.1$ and $\alpha= 0.5$, the $\alpha$-GAN converges faster. 
This is consistent with the observation in Section~\ref{sec:Alpha} that the gradient magnitude is accelerated in range $\alpha \in (0,1)$, which results in a faster learning speed. 

Figure~\ref{fig:1DGaussianAna} displays the final value of the discriminator output $D(x)$ for all values of $\alpha$. As explained in Section~\ref{sec:sys}, if the $\alpha$-GAN model is trained well, it should be able to generate samples having probability distribution almost equal to the true population. In this case, the discriminator is not able to discriminate sample source, and we should have $D(x) = 0.5$ for all samples $x$. It can be seem that at $\alpha = 0.1$, $D(x)$ vs. $x$ is flattened, which remains $0.5$ for almost all $x$. 
We also plot the discriminator loss in Figure~\ref{fig:1DGaussianAna}. The loss values are in different scales when $\alpha$ changes. However, it is clear that $\alpha = 0.1$ generates a stable convergence curve. 
More results can be found in the supplementary materials.

\begin{figure}[t]
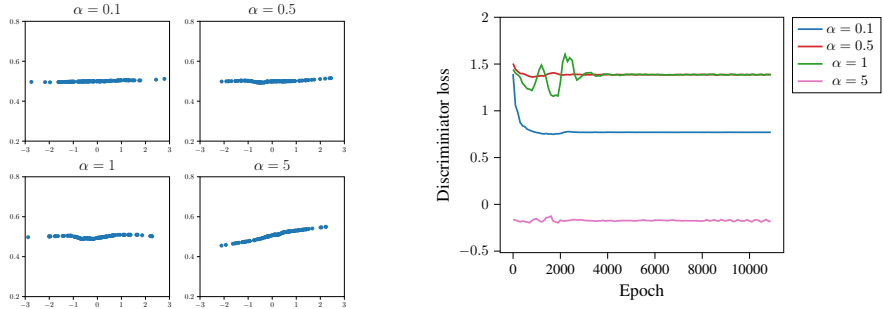

	\begin{minipage}{0.45\textwidth}
		\centering
		\scalebox{0.28}{\input{figures/1DGaussD_0.1_10000.tex}}
		\scalebox{0.28}{\input{figures/1DGaussD_0.5_10000.tex}}\\
		\scalebox{0.28}{\input{figures/1DGaussD_1_10000.tex}}
		\scalebox{0.28}{\input{figures/1DGaussD_5_10000.tex}}
	\end{minipage}
	\begin{minipage}{0.45\textwidth}
		\centering
		\scalebox{0.55}{\input{figures/1DGaussLoss_0.1.tex}}
	\end{minipage}
	\caption{Left: the final discriminator's soft decision $D(x)$ after $10^4$ epochs for different $\alpha$; Right: discriminator's loss vs epoch index.}
	\label{fig:1DGaussianAna}
\end{figure}

\paragraph{MNIST}

We set batch size to $64$ and latent dimension to $100$. 
Adam optimizer is used with learning rate $0.0002$, $\beta_1 = 0.5$ and $\beta_2=0.999$. 
Images are $ 28\times28$ gray-scale. There are one flatten layer and two fully connected layers with ReLu activation. 
We try R\'{e}nyi order values $\alpha \in \Set{0.05,1,2,3}$ and run $100$ epochs. Figure~\ref{fig:MNIST} shows the generated images for all $\alpha$. The same as previous experiments on 1D Gaussian, $\alpha=0.05$ produces the best performance. The results is deteriorating as $\alpha$ increases. We find that there is a model collapse when $\alpha > 3$. This could mean that the GAN model is more robust for smaller $\alpha$.

\begin{figure}[h]
	\subfloat[$\alpha=0.05$]{
		\begin{tabular}[b]{c}
			\includegraphics[width=0.055\textwidth]{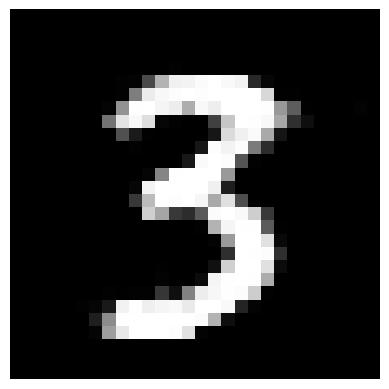}
			\includegraphics[width=0.055\textwidth]{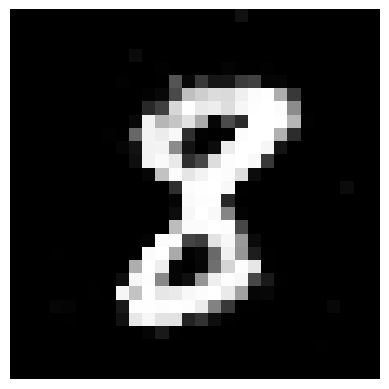}
			\includegraphics[width=0.055\textwidth]{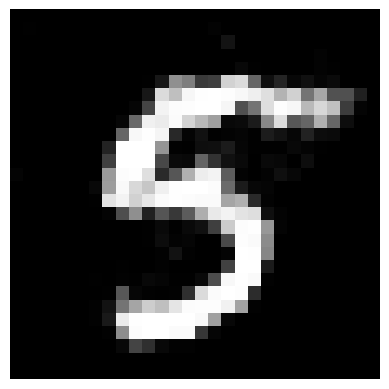}
			\includegraphics[width=0.055\textwidth]{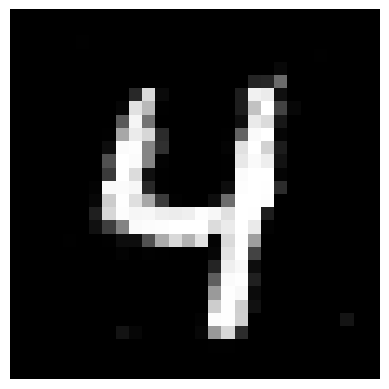}
			\includegraphics[width=0.055\textwidth]{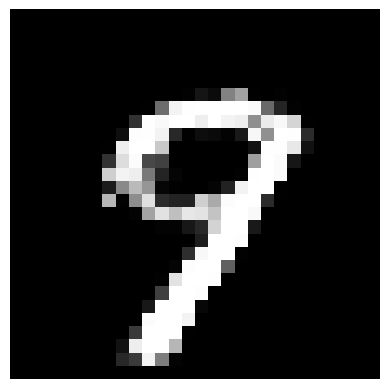}
			\includegraphics[width=0.055\textwidth]{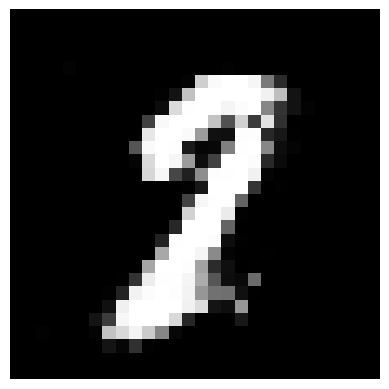}
			\includegraphics[width=0.055\textwidth]{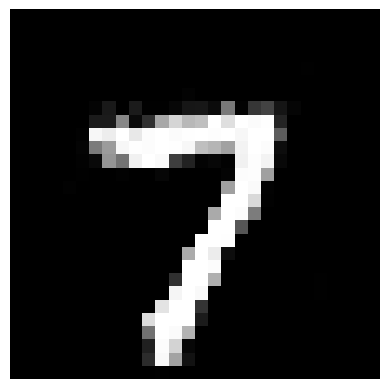}
			\includegraphics[width=0.055\textwidth]{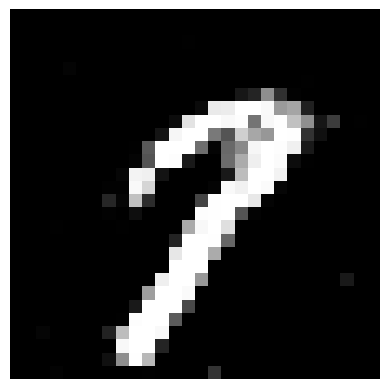}
			\includegraphics[width=0.055\textwidth]{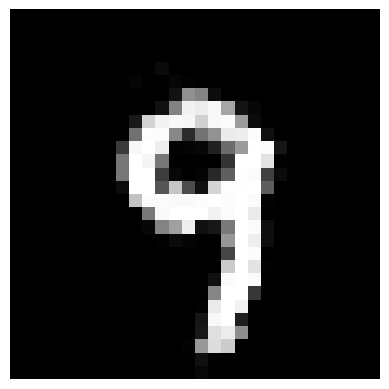}
			\includegraphics[width=0.055\textwidth]{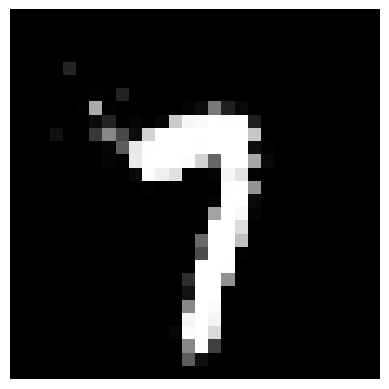}
			\includegraphics[width=0.055\textwidth]{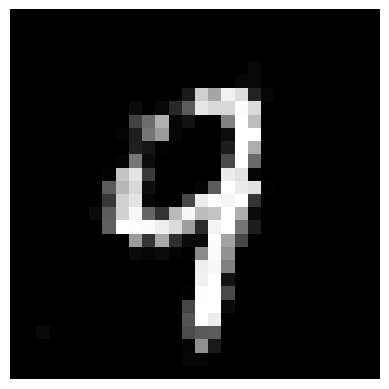}
			\includegraphics[width=0.055\textwidth]{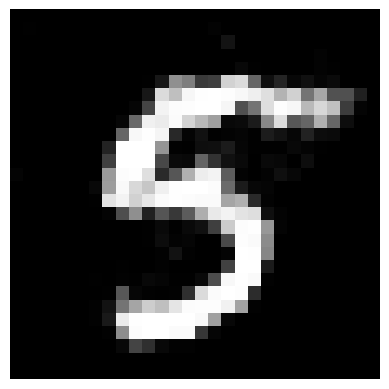}
			\includegraphics[width=0.055\textwidth]{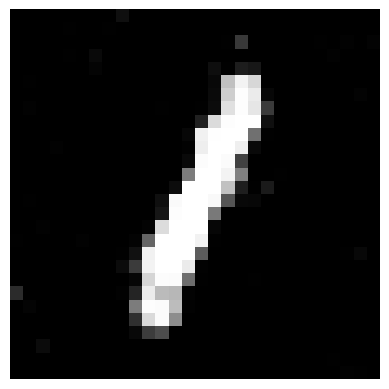}
			\includegraphics[width=0.055\textwidth]{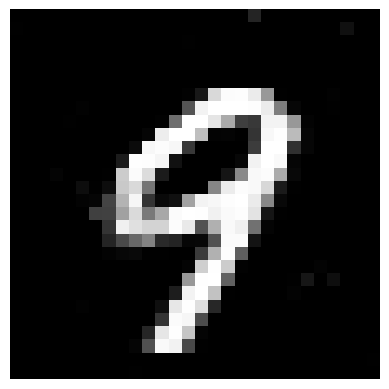}
			\includegraphics[width=0.055\textwidth]{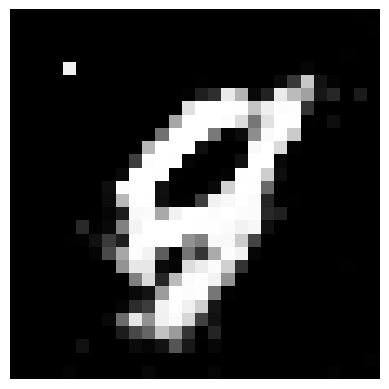}
			\includegraphics[width=0.055\textwidth]{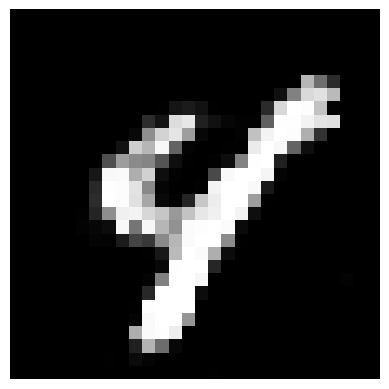}\\
			\includegraphics[width=0.055\textwidth]{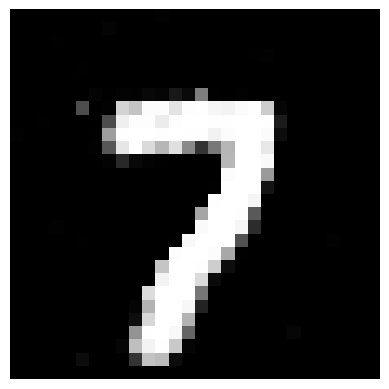}
			\includegraphics[width=0.055\textwidth]{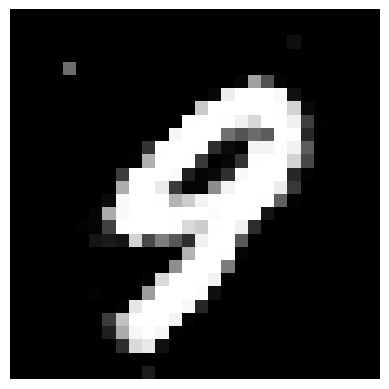}
			\includegraphics[width=0.055\textwidth]{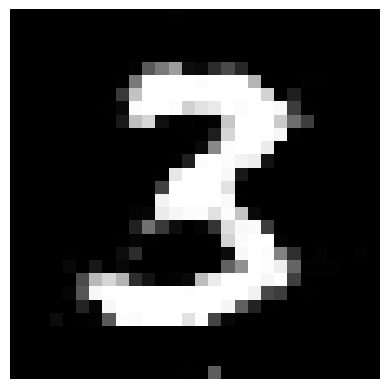}
			\includegraphics[width=0.055\textwidth]{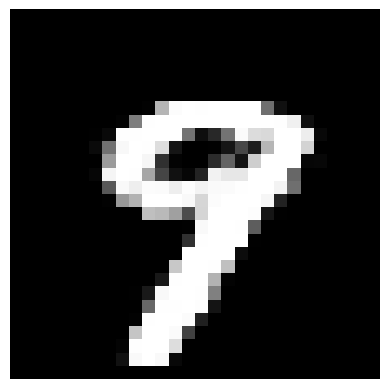}
			\includegraphics[width=0.055\textwidth]{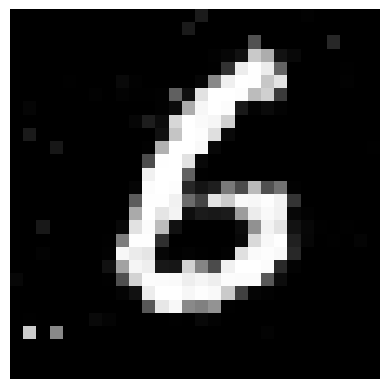}
			\includegraphics[width=0.055\textwidth]{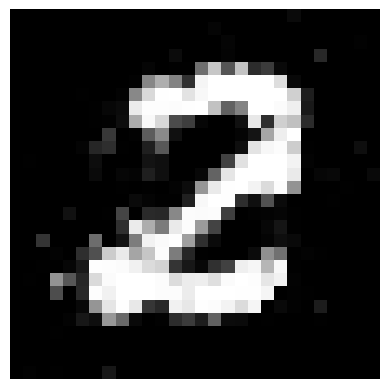}
			\includegraphics[width=0.055\textwidth]{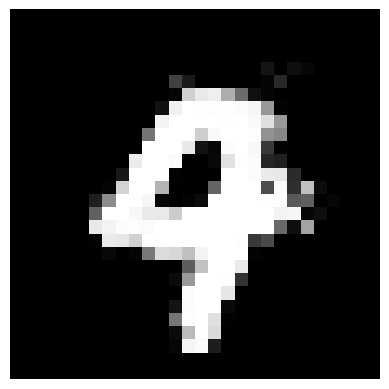}
			\includegraphics[width=0.055\textwidth]{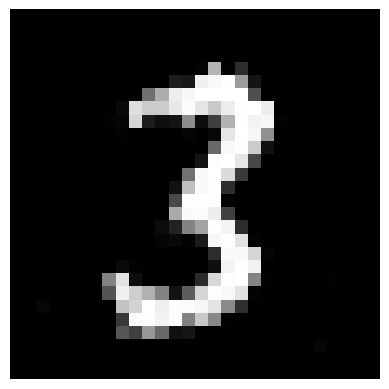}
			\includegraphics[width=0.055\textwidth]{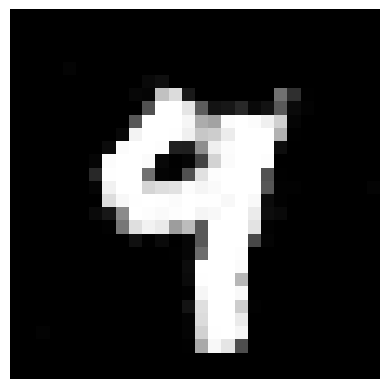}
			\includegraphics[width=0.055\textwidth]{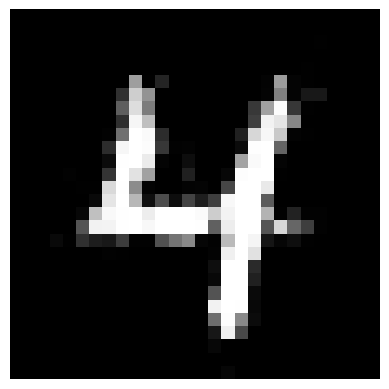}
			\includegraphics[width=0.055\textwidth]{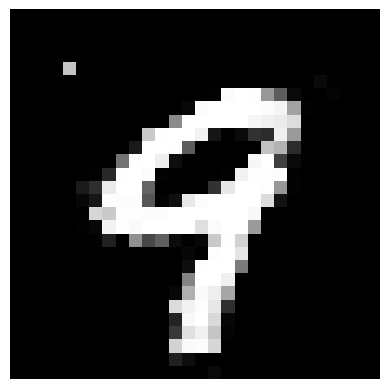}
			\includegraphics[width=0.055\textwidth]{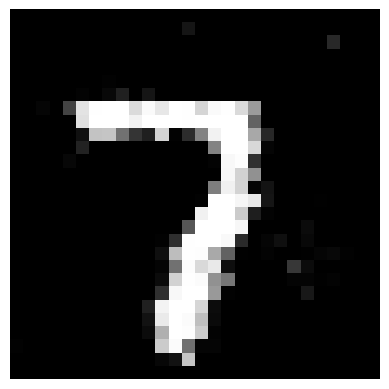}
			\includegraphics[width=0.055\textwidth]{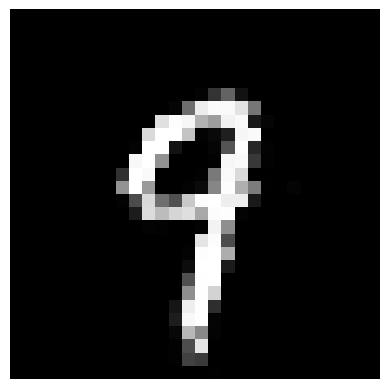}
			\includegraphics[width=0.055\textwidth]{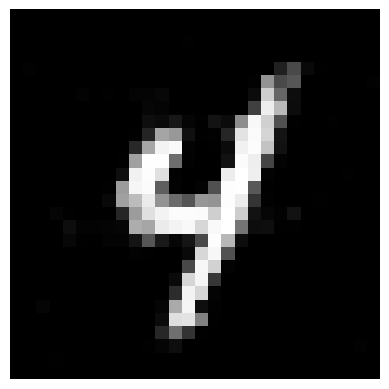}
			\includegraphics[width=0.055\textwidth]{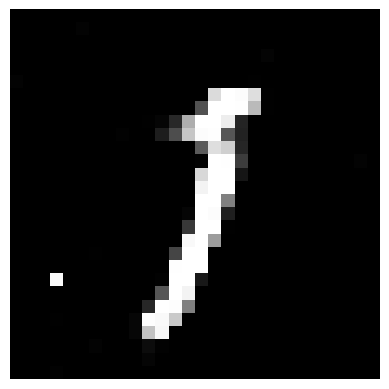}
			\includegraphics[width=0.055\textwidth]{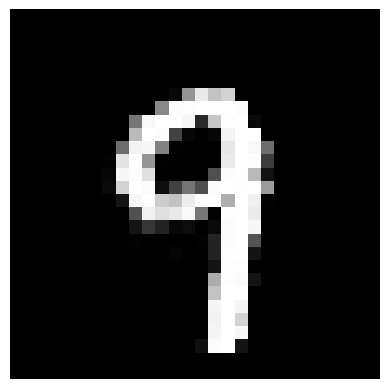}
		\end{tabular}
	}\\
		\subfloat[$\alpha=1$]{
		\begin{tabular}[b]{c}
			\includegraphics[width=0.055\textwidth]{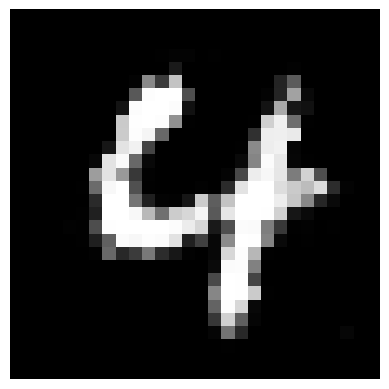}
			\includegraphics[width=0.055\textwidth]{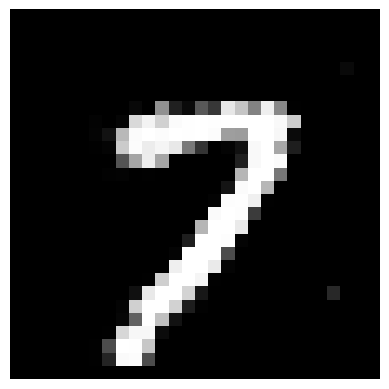}
			\includegraphics[width=0.055\textwidth]{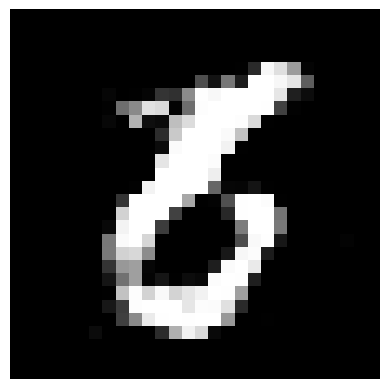}
			\includegraphics[width=0.055\textwidth]{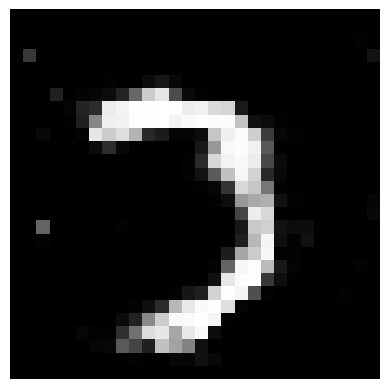}
			\includegraphics[width=0.055\textwidth]{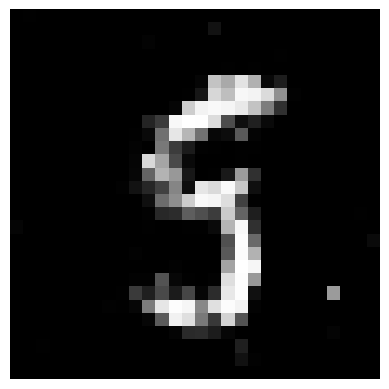}
			\includegraphics[width=0.055\textwidth]{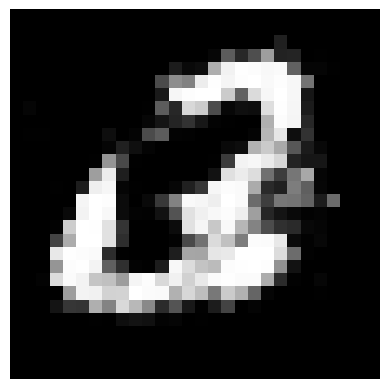}
			\includegraphics[width=0.055\textwidth]{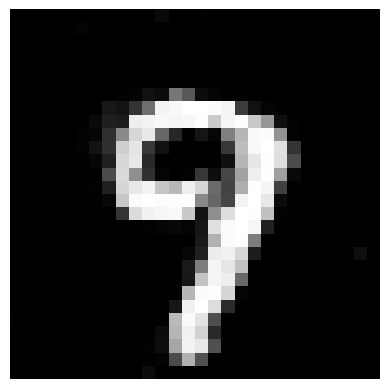}
			\includegraphics[width=0.055\textwidth]{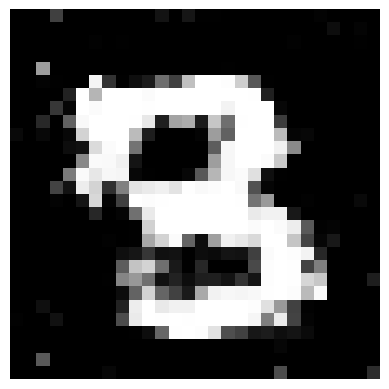}
			\includegraphics[width=0.055\textwidth]{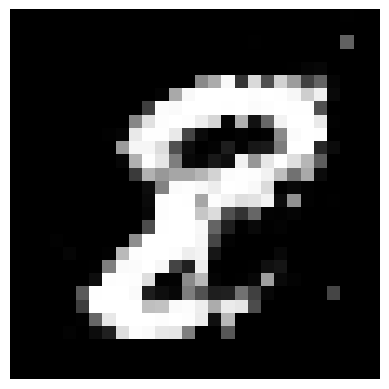}
			\includegraphics[width=0.055\textwidth]{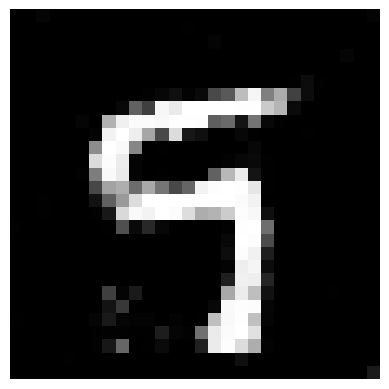}
			\includegraphics[width=0.055\textwidth]{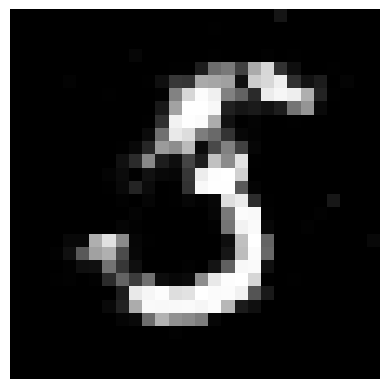}
			\includegraphics[width=0.055\textwidth]{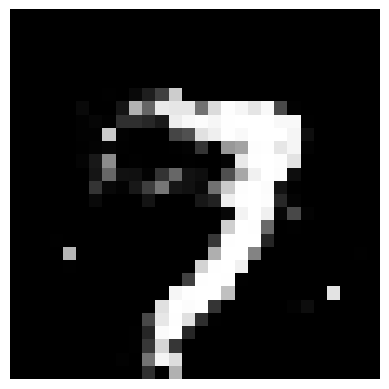}
			\includegraphics[width=0.055\textwidth]{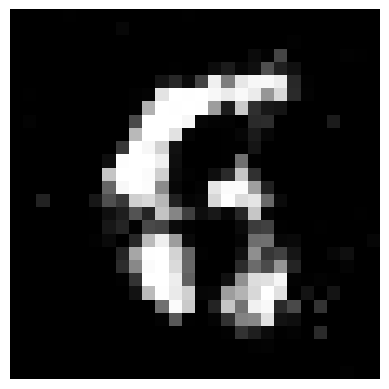}
			\includegraphics[width=0.055\textwidth]{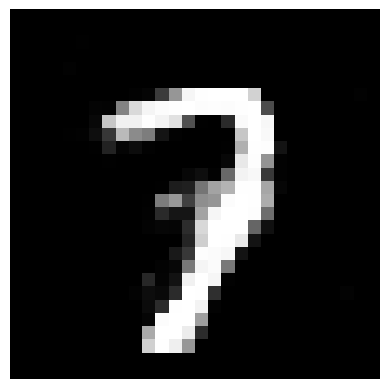}
			\includegraphics[width=0.055\textwidth]{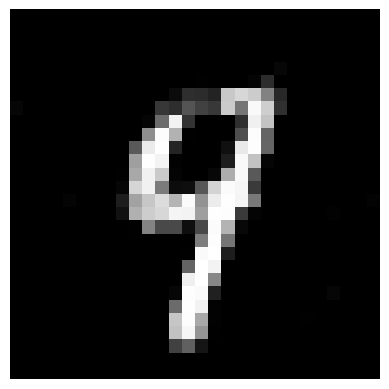}
			\includegraphics[width=0.055\textwidth]{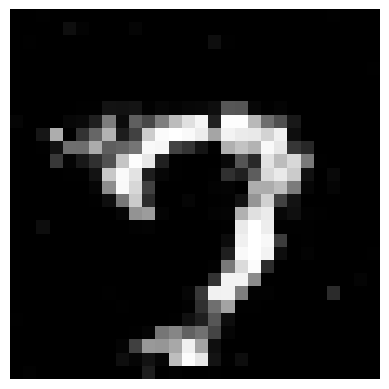}\\
			\includegraphics[width=0.055\textwidth]{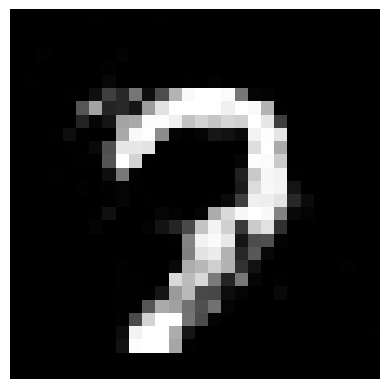}
			\includegraphics[width=0.055\textwidth]{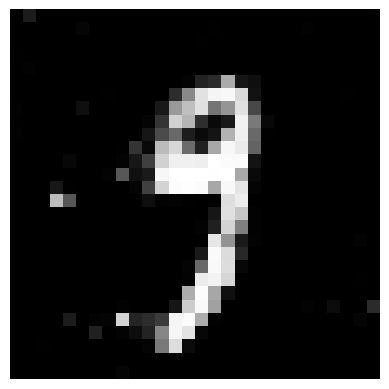}
			\includegraphics[width=0.055\textwidth]{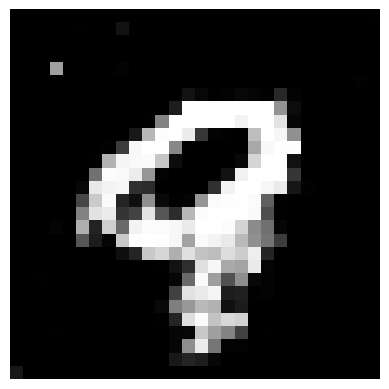}
			\includegraphics[width=0.055\textwidth]{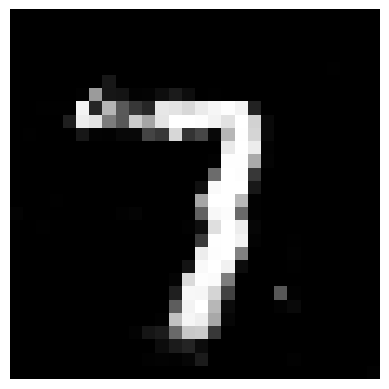}
			\includegraphics[width=0.055\textwidth]{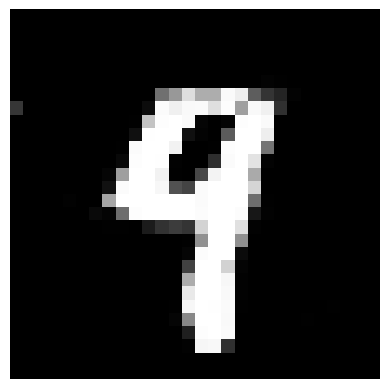}
			\includegraphics[width=0.055\textwidth]{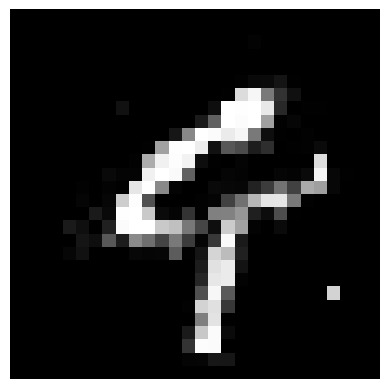}
			\includegraphics[width=0.055\textwidth]{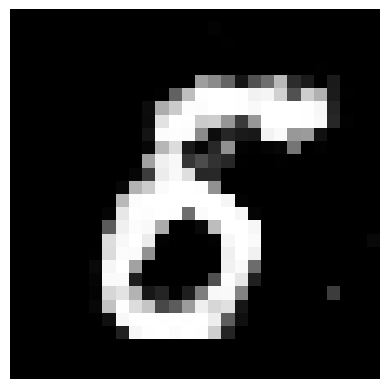}
			\includegraphics[width=0.055\textwidth]{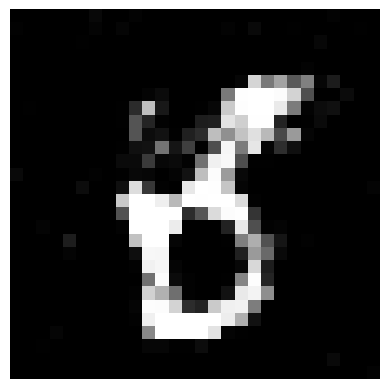}
			\includegraphics[width=0.055\textwidth]{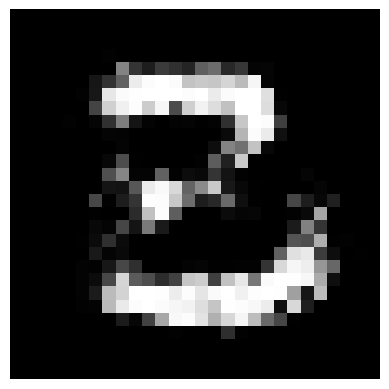}
			\includegraphics[width=0.055\textwidth]{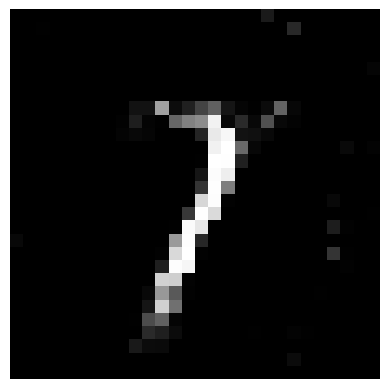}
			\includegraphics[width=0.055\textwidth]{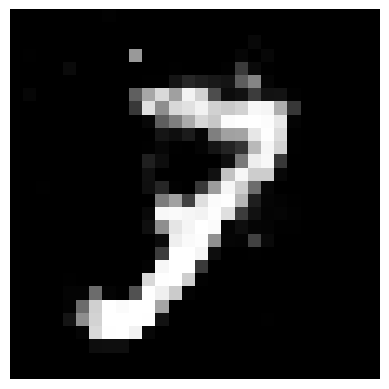}
			\includegraphics[width=0.055\textwidth]{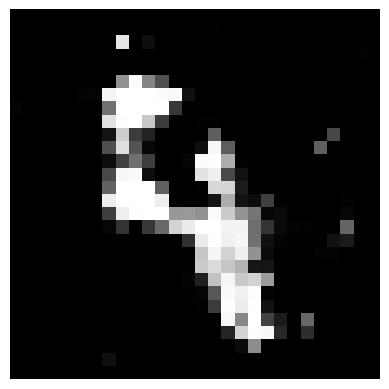}
			\includegraphics[width=0.055\textwidth]{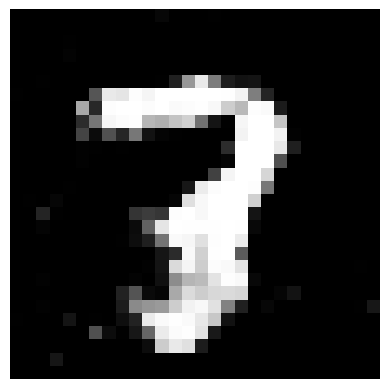}
			\includegraphics[width=0.055\textwidth]{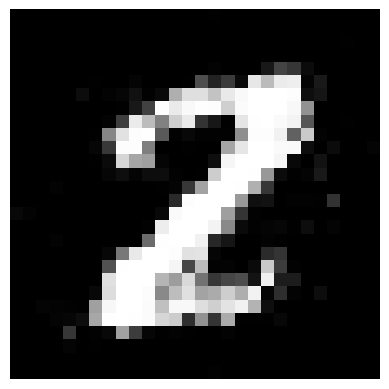}
			\includegraphics[width=0.055\textwidth]{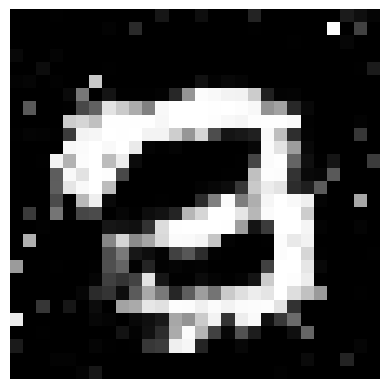}
			\includegraphics[width=0.055\textwidth]{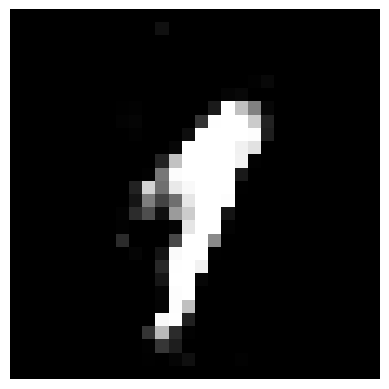}
		\end{tabular}
	}\\
		\subfloat[$\alpha=2$]{
		\begin{tabular}[b]{c}
			\includegraphics[width=0.055\textwidth]{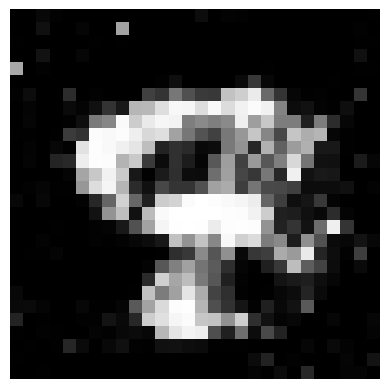}
			\includegraphics[width=0.055\textwidth]{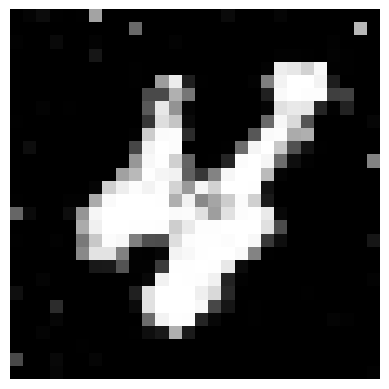}
			\includegraphics[width=0.055\textwidth]{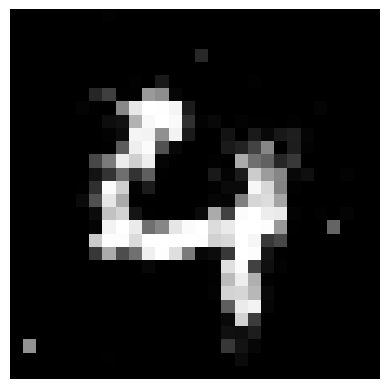}
			\includegraphics[width=0.055\textwidth]{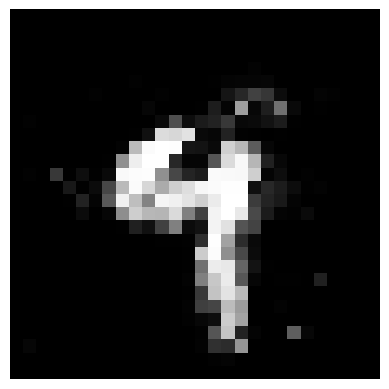}
			\includegraphics[width=0.055\textwidth]{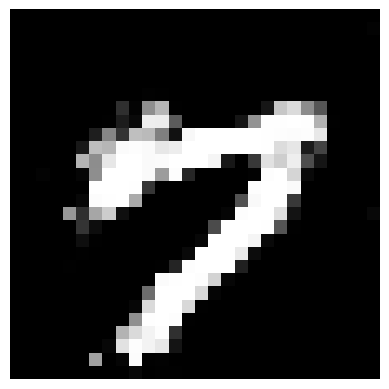}
			\includegraphics[width=0.055\textwidth]{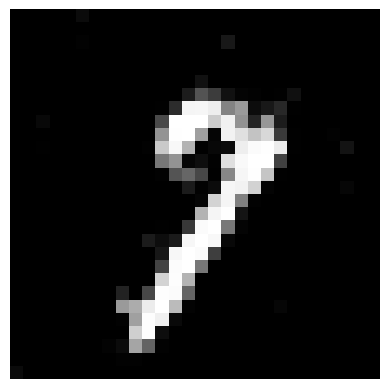}
			\includegraphics[width=0.055\textwidth]{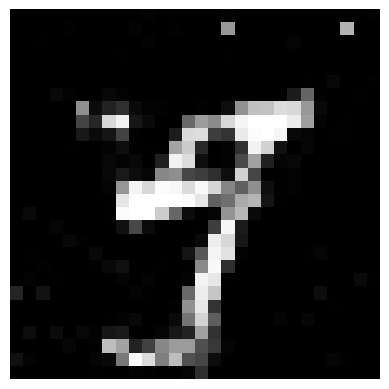}
			\includegraphics[width=0.055\textwidth]{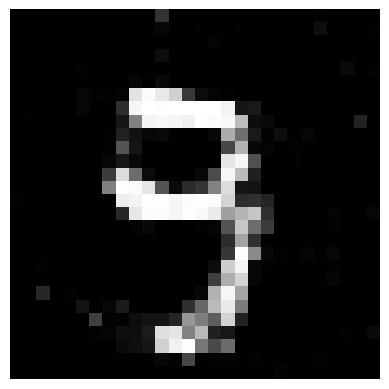}
			\includegraphics[width=0.055\textwidth]{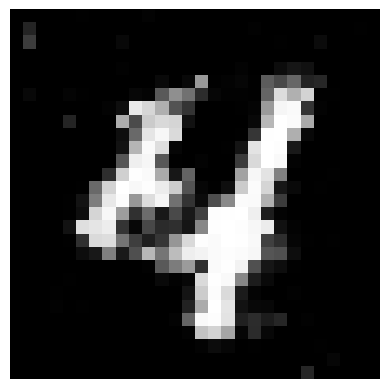}
			\includegraphics[width=0.055\textwidth]{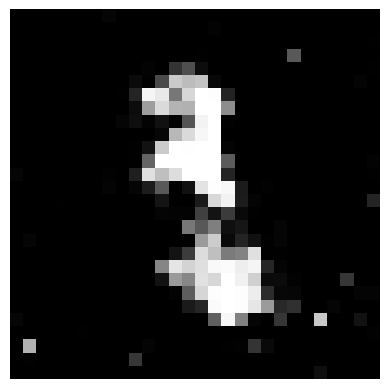}
			\includegraphics[width=0.055\textwidth]{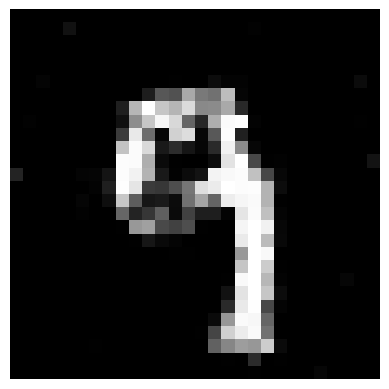}
			\includegraphics[width=0.055\textwidth]{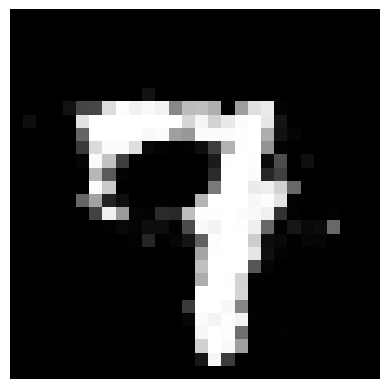}
			\includegraphics[width=0.055\textwidth]{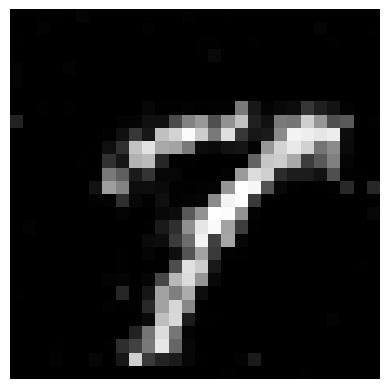}
			\includegraphics[width=0.055\textwidth]{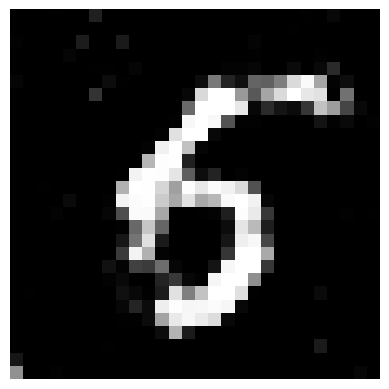}
			\includegraphics[width=0.055\textwidth]{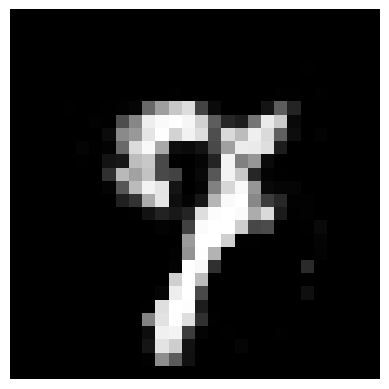}
			\includegraphics[width=0.055\textwidth]{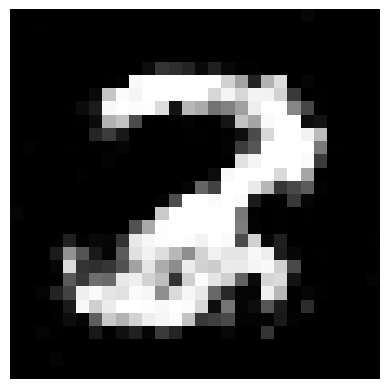}\\
			\includegraphics[width=0.055\textwidth]{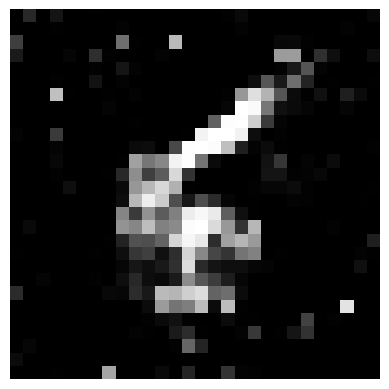}
			\includegraphics[width=0.055\textwidth]{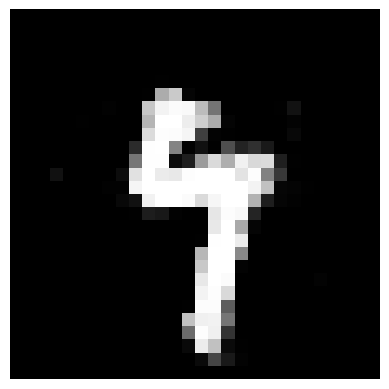}
			\includegraphics[width=0.055\textwidth]{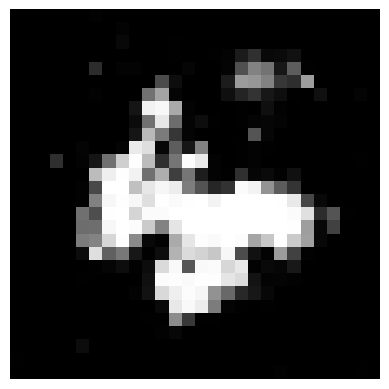}
			\includegraphics[width=0.055\textwidth]{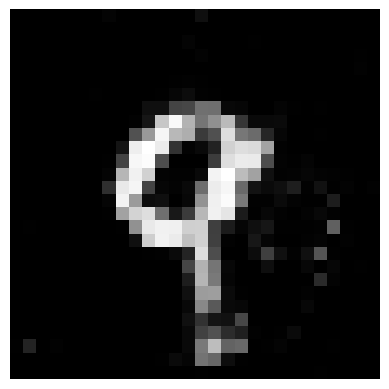}
			\includegraphics[width=0.055\textwidth]{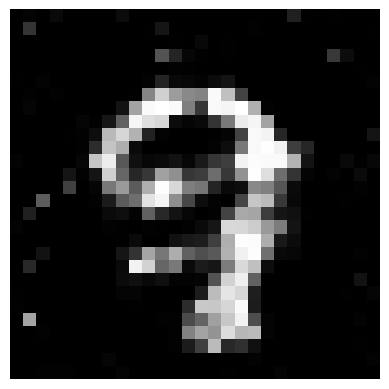}
			\includegraphics[width=0.055\textwidth]{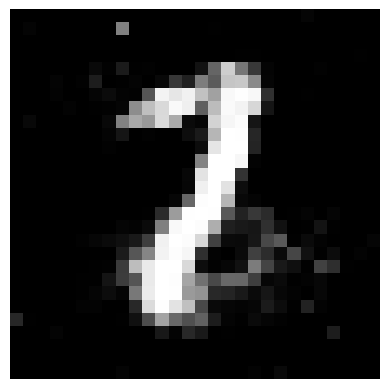}
			\includegraphics[width=0.055\textwidth]{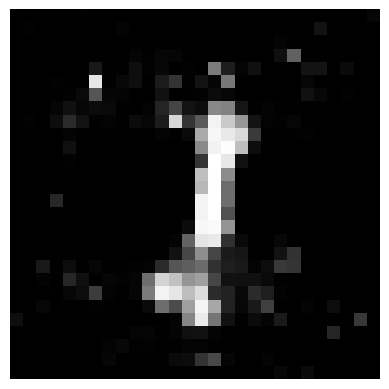}
			\includegraphics[width=0.055\textwidth]{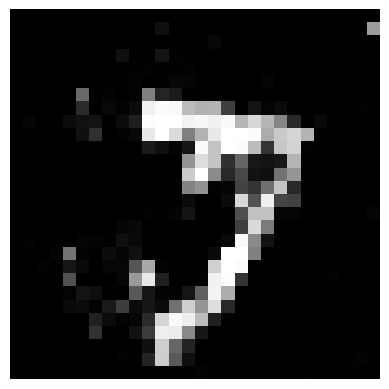}
			\includegraphics[width=0.055\textwidth]{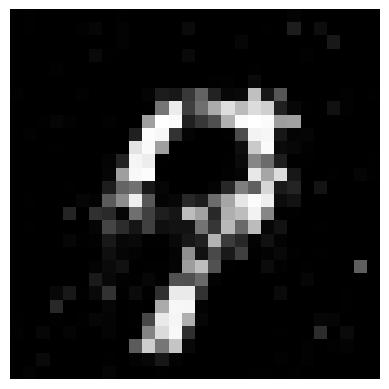}
			\includegraphics[width=0.055\textwidth]{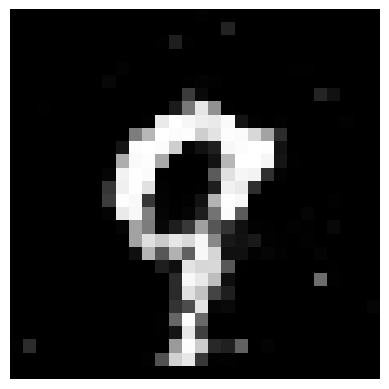}
			\includegraphics[width=0.055\textwidth]{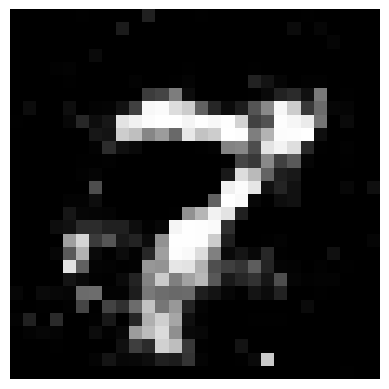}
			\includegraphics[width=0.055\textwidth]{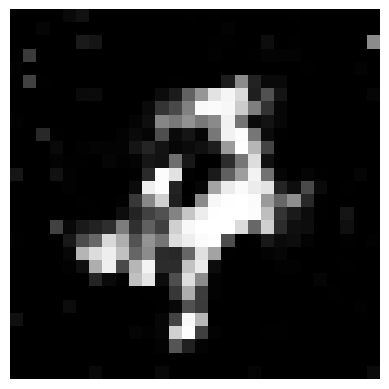}
			\includegraphics[width=0.055\textwidth]{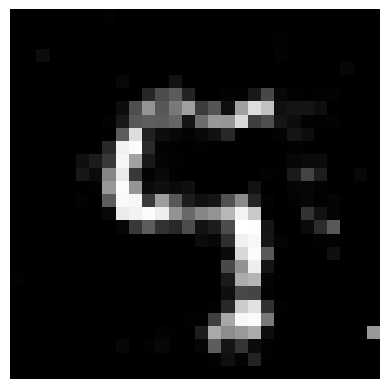}
			\includegraphics[width=0.055\textwidth]{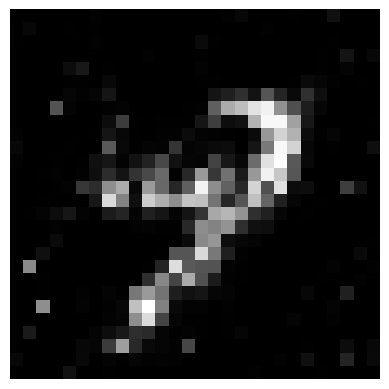}
			\includegraphics[width=0.055\textwidth]{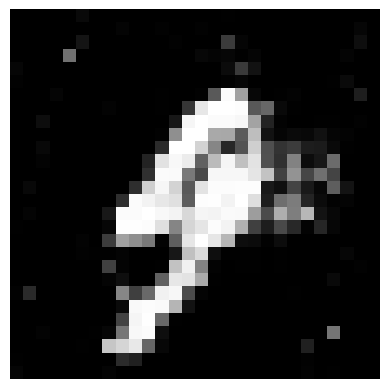}
			\includegraphics[width=0.055\textwidth]{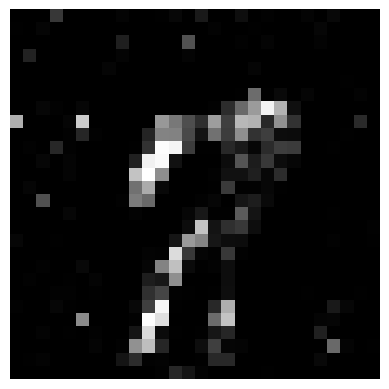}
		\end{tabular}
	}\\
	\subfloat[$\alpha=5$]{
	\begin{tabular}[b]{c}
		\includegraphics[width=0.055\textwidth]{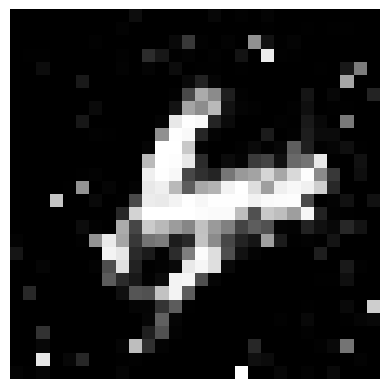}
		\includegraphics[width=0.055\textwidth]{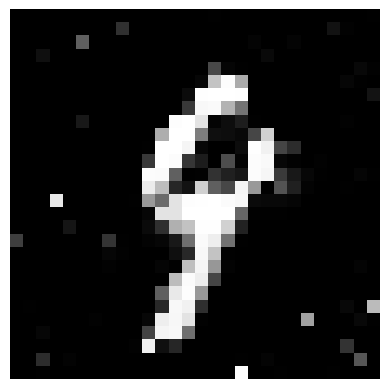}
		\includegraphics[width=0.055\textwidth]{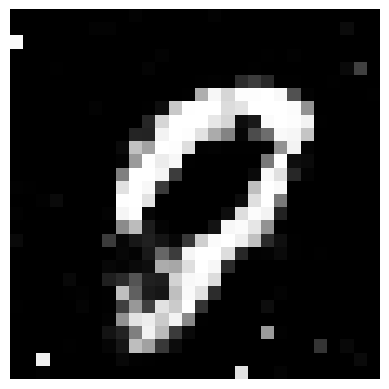}
		\includegraphics[width=0.055\textwidth]{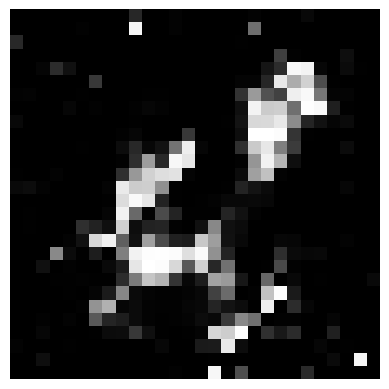}
		\includegraphics[width=0.055\textwidth]{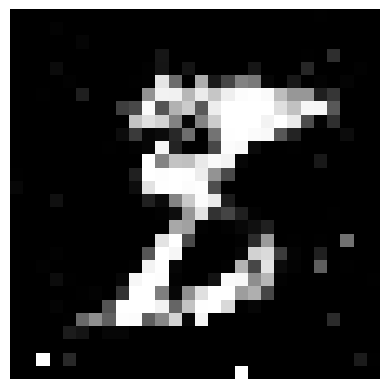}
		\includegraphics[width=0.055\textwidth]{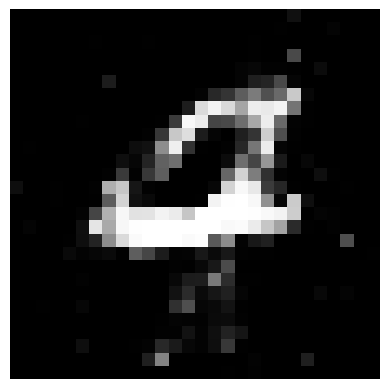}
		\includegraphics[width=0.055\textwidth]{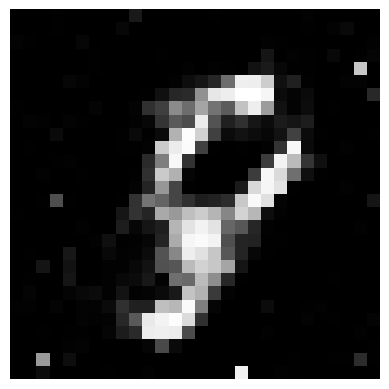}
		\includegraphics[width=0.055\textwidth]{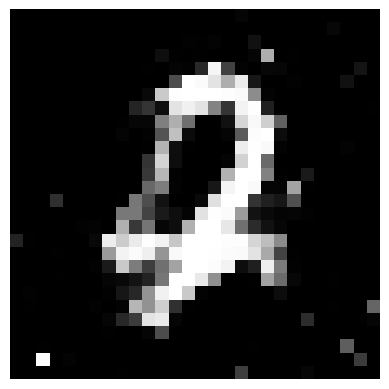}
		\includegraphics[width=0.055\textwidth]{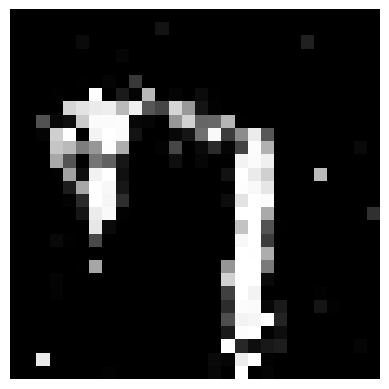}
		\includegraphics[width=0.055\textwidth]{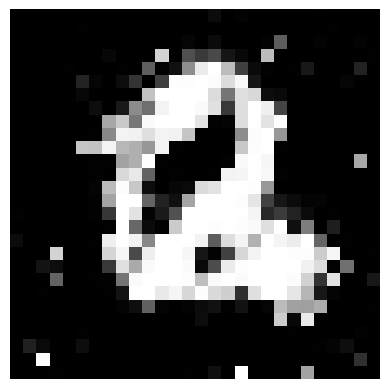}
		\includegraphics[width=0.055\textwidth]{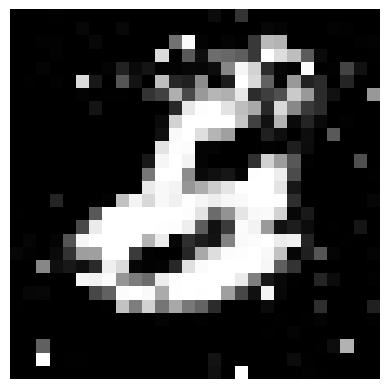}
		\includegraphics[width=0.055\textwidth]{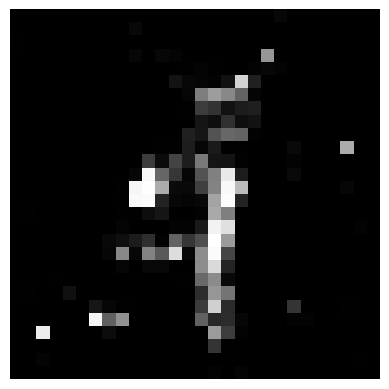}
		\includegraphics[width=0.055\textwidth]{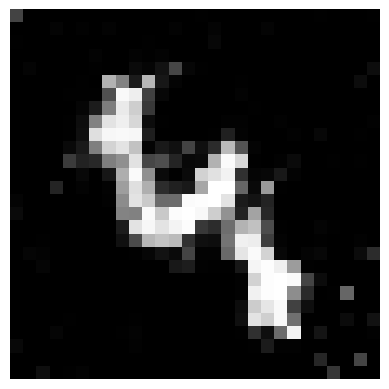}
		\includegraphics[width=0.055\textwidth]{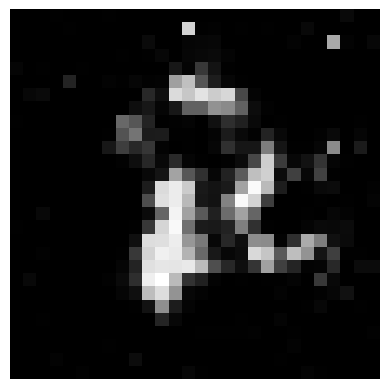}
		\includegraphics[width=0.055\textwidth]{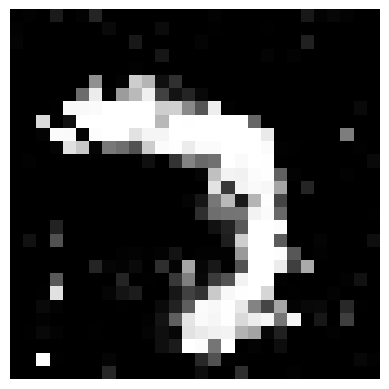}
		\includegraphics[width=0.055\textwidth]{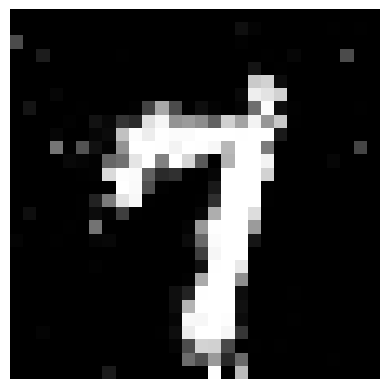}\\
		\includegraphics[width=0.055\textwidth]{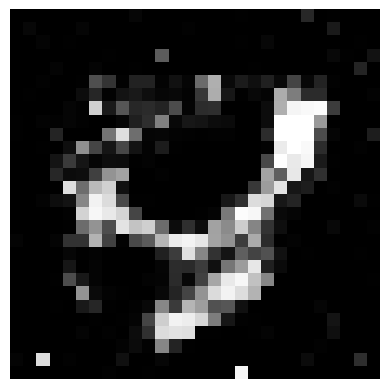}
		\includegraphics[width=0.055\textwidth]{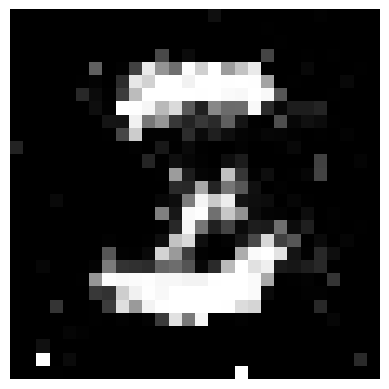}
		\includegraphics[width=0.055\textwidth]{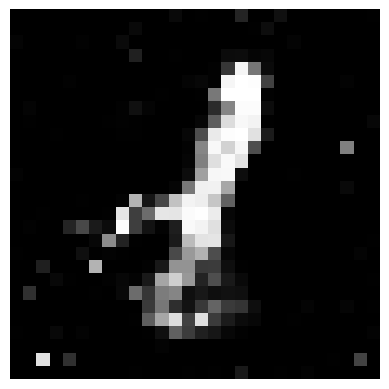}
		\includegraphics[width=0.055\textwidth]{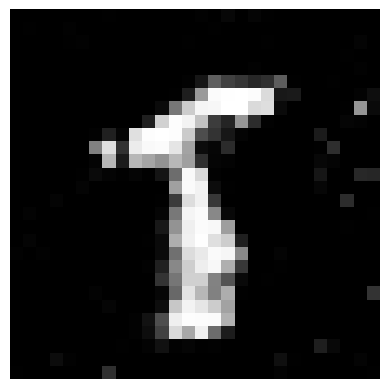}
		\includegraphics[width=0.055\textwidth]{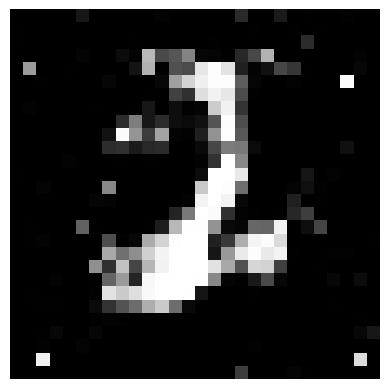}
		\includegraphics[width=0.055\textwidth]{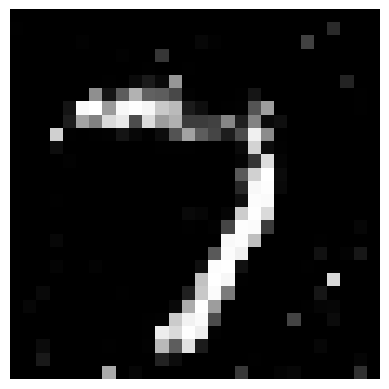}
		\includegraphics[width=0.055\textwidth]{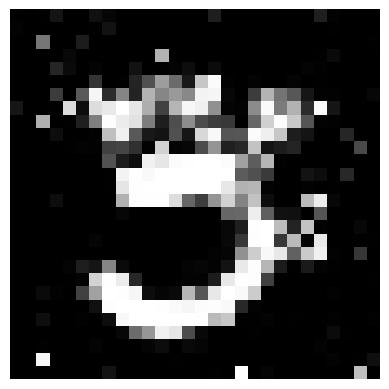}
		\includegraphics[width=0.055\textwidth]{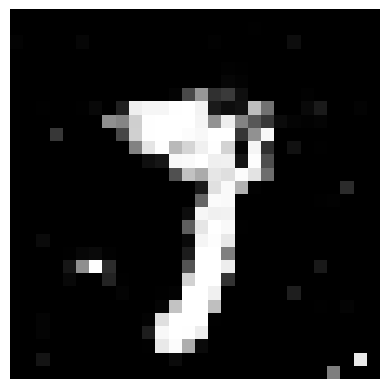}
		\includegraphics[width=0.055\textwidth]{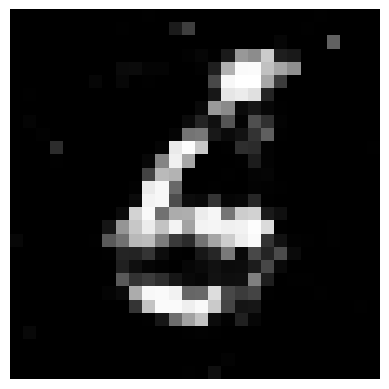}
		\includegraphics[width=0.055\textwidth]{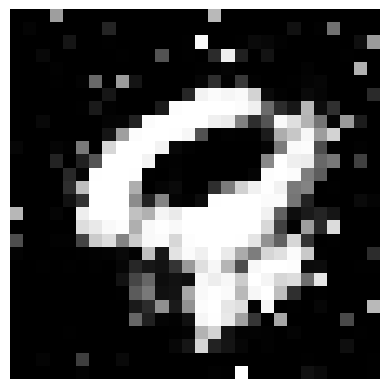}
		\includegraphics[width=0.055\textwidth]{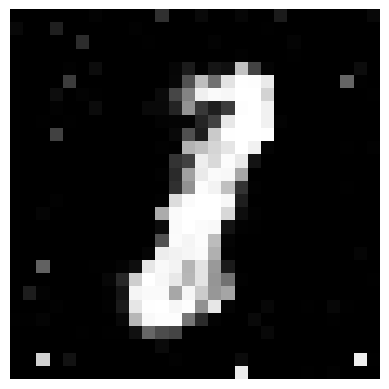}
		\includegraphics[width=0.055\textwidth]{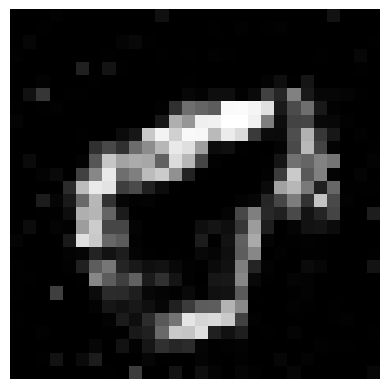}
		\includegraphics[width=0.055\textwidth]{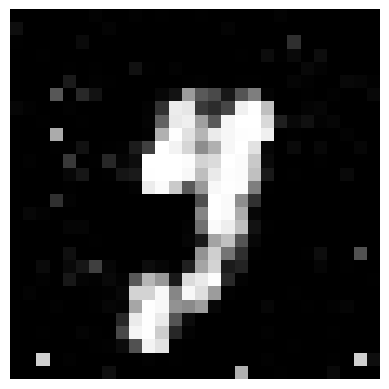}
		\includegraphics[width=0.055\textwidth]{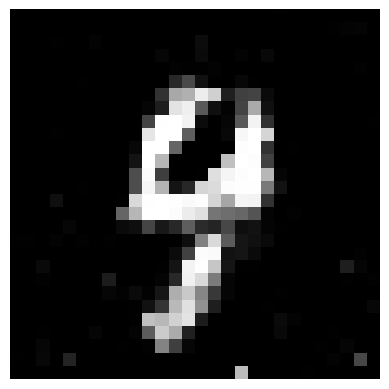}
		\includegraphics[width=0.055\textwidth]{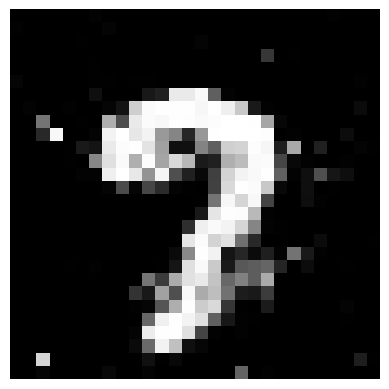}
		\includegraphics[width=0.055\textwidth]{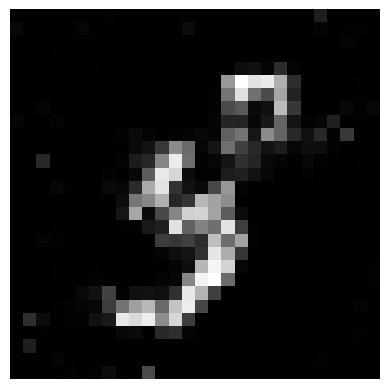}
	\end{tabular}
	}
	\caption{MNIST digits generated by $\alpha$-GAN for different $\alpha$.}
	\label{fig:MNIST}
\end{figure}

\section{Related Work}
\label{sec:relate}
Before this paper, there have been several attempts to use R\'{e}nyi measures to generalize the vanilla GAN~\cite{Goodfellow2014_VanillaGAN}.
In~\cite{Bhatia2021_KthOrder_Renyi_GAN}, R\'{e}nyi cross entropy is defined as $H_\alpha (P_{X}, \hat{P}_{X}) = \frac{1}{1-\alpha} \log \E_{X \sim P_{X}} [ ( \hat{P}_{X}^{-1}(X) )^{1-\alpha}  ]$ based on the fact that the R\'{e}nyi entropy is an $f$-mean measure for $f(t) = \exp((1-\alpha)t)$.
The main issue with this definition is that $\min_{\hat{P}_{X}} H_\alpha (P_{X}, \hat{P}_{X}) \neq H_{\alpha}(P_{X})$ and yet the resulting optimal discriminator $D^*(x)$ at $\alpha=1$ does not coincide with the vanilla GAN. 
This is why this R\'{e}nyi cross entropy is only used for generator loss function, not the discriminator, which is defined as $V_{\alpha}(D, P_g) = -H_\alpha (P_r, D(X)) -H_\alpha (P_g, 1-D(X)) $ in a sense close to the Jensen-Shannon divergence interpretation, as in vanilla GAN~\cite{Goodfellow2014_VanillaGAN}, not the conditional cross entropy. There are some experimental results in~\cite{Bhatia2021_KthOrder_Renyi_GAN} showing the performance in range $\alpha \in (0,3)$. But, it is not specific to the region $\alpha \in (0,1)$ and there is no analysis on this aspect. 

In~\cite{Sarraf2021_RGAN}, R\'{e}nyi cross entropy is defined as $H_\alpha (P_{X}, \hat{P}_{X}) = \frac{1}{1-\alpha} \log \E_{X \sim P_{X_\alpha}} [ ( \hat{P}_{X}^{-1}(X) )^{1-\alpha}  ]$, where $P_{X_\alpha}(x) = \frac{P_{X}^\alpha(x)}{\sum_{x} P_{X}^\alpha(x)}$ is the $\alpha$-scaled probability. This definition results in a decomposition $H_\alpha (P_{X}, \hat{P}_{X}) = D_\alpha (P_{X}, \hat{P}_{X}) + H_\alpha (P_{X})$ such that $\min_{\hat{P}_{X}} H_\alpha (P_{X}, \hat{P}_{X}) = H_{\alpha}(P_{X})$. The value function $V_{\alpha}(D, P_g) = -H_\alpha (P_r, D(X)) -H_\alpha (P_g, 1-D(X)) $ is used for training both discriminator and generator. It is shown that value $\alpha = 0.1$ generates better FID score. However, there is no explanation for this performance improvement. This result actually implies the intriguing region $\alpha \in (0,1)$ is worth exploring. 

On the other hand, the Jensen-Shannon divergence has been generalized to $f$-divergence in~\cite{Nowozin2016_fGAN} to propose the $f$-GAN. This is based on a variational representation of $f$-divergence $D_{f}(P_X, \hat{P}_{X}) = \int \hat{P}_{X}(x) \sup_{t \in \text{dom}_{f^*}} \Set{ t\frac{P_{X}(x)}{\hat{P}_{X}(x)} - f^*(t) } \dif x$ discovered in~\cite{Nguyen2010_FDiv}, where $f^*$ is the Fenchel conjugate of $f$. There are various options of $f$ function, e.g., Kullback-Leibler divergence, Pearson $\chi^2$, etc. The $f$ function in $f$-GAN is supposed to be chosen in advance, which may affect the performance. However, it is not sure which $f$ can produce better learning speed. On the other hand, choosing $f$ as the $\alpha$-divergence will result in a loss function parameterized by $\alpha$. But, it is different from R\'{e}nyi divergence, in particular, in region $\alpha \in (0,1)$.  

Instead of probability divergence, the authors in~\cite{Arjovsky2017_WGAN} adopted the Wasserstein metric $W_\alpha$ (more precisely, the earth mover distance $W_1$) and formulated a WGAN.  
The reason for WGAN, as explained in~\cite[Example~1]{Arjovsky2017_WGAN}, is that the learning will not proceed by taking the gradient in the probability space in some special cases. 
However, it is pointed out in~\cite{Gulrajani2017_ImprWGAN} that the weight clipping in WGAN may affect the convergence. Unlike WGAN, the $\alpha$-GAN proposed in this paper still measures distance or divergence between two probabilities (by R\'{e}nyi cross entropy).  
Though Wasserstein metric is not defined in $\alpha \in (0,1)$, the study in this paper suggests order lower than $1$ or high order Wasserstein is worth exploring. 

Recent studies in~\cite{Welfert2024,Veiner2024AlphaGAN,Kurri2022ISIT_AlphaGAN,Kurri2021ITW_AlphaGAN} propose $\alpha$-GAN with value function $V_\alpha(D, P_g) = \frac{\alpha}{\alpha-1} ( \E_{X \sim P_r} [D^{\frac{\alpha-1}{\alpha}}(X)]  + \E_{X \sim P_g} [(1-D(x))^{\frac{\alpha-1}{\alpha}}] )$.
This definition is borrowed from an information leakage measure in~\cite{Liao2019_AlphaLeak} based on a statistical estimation setting. It generalizes the class probability estimation (CPE) loss for binary classification \cite{Reid2010_CPE} and coincides with the Arimoto divergence, an $f$-divergence measure. 
The Arimoto divergence was first proposed in \cite{Osterreicher1996} for region $\alpha \in [1,\infty)$ only, and later extended to $\alpha \in (0,\infty)$~\cite{Liese2006}.  
As the objective function is not defined at $\alpha = 0$, this value is specifically left out.  
The studies~\cite{Welfert2024,Veiner2024AlphaGAN,Kurri2022ISIT_AlphaGAN,Kurri2021ITW_AlphaGAN} show that tuning parameter $\alpha$ may help deal with model collapse and vanishing gradient. But, there is no specific exploration on region $\alpha \in (0,1)$. 
One of the reason could be that the value function is not well defined as a mean measure. In fact, the coefficient $\frac{\alpha}{\alpha-1}$ should be exponential index instead, e.g., $V_\alpha(D, P_g) = ( \E_{X \sim P_r} [D^{\frac{\alpha-1}{\alpha}}(X)] )^{\frac{\alpha}{\alpha-1}}  + ( \E_{X \sim P_g} [(1-D(X))^{\frac{\alpha-1}{\alpha}}] )^{\frac{\alpha}{\alpha-1}}$, so that it is a weighted power mean measure well defined in the overall R\'{e}nyi order range $\alpha \in [0,\infty)$. Taking the logarithm of this measure is exactly the R\'{e}nyi cross entropy we used in this paper.

\section{Discussion}
\label{sec:discuss}

This paper shows one of the advantages for introducing R\'{e}nyi order $\alpha$ to the loss function in GAN: accelerated gradient. 
As the cross entropy is widely used in machine learning, the results, e.g., function~\eqref{eq:powermean} can be extended to other problems. The logarithm can be omitted to make the optimization easier. For example, define the $\alpha$-loss function by $L_{\alpha}(P_{Y|X}, \hat{P}_{Y|X}) = ( \E_{XY \sim P_{XY}} [  \hat{P}_{Y|X}^{\frac{\alpha-1}{\alpha}} (Y|X)  ]  )^{\frac{\alpha}{1-\alpha} } = (\frac{1}{N} \sum_{x,y} P_{Y|X}(y|x) \hat{P}_{Y|X}^\frac{\alpha-1}{\alpha}(y|x) )^{\frac{\alpha}{1-\alpha}}$, where $y$ is the label and $x$ is the sample. The decision variable $\hat{P}_{Y|X}$ can be sigmoid or softmax. 
Knowing $\hat{P}_{Y|X}^{-1}(y|x)$ measures the uncertainty if sample $x$ is classified to label $y$, $L_{\alpha}(P_{Y|X}, \hat{P}_{Y|X})$ is the $\frac{1-\alpha}{\alpha}$-order weighted power mean of uncertainty incurred by $\hat{P}_{Y|X}$ wrt the true or empirical distribution $P_{Y|X}$. 
The logarithm of $L_{\alpha}(P_{Y|X}, \hat{P}_{Y|X})$ is the R\'{e}nyi cross entropy, which becomes Shannon cross entropy at $\alpha = 1$.

There is an attempt in~\cite{Sypherd2019,Sypherd2022} where $\alpha$-loss is defined as $L_{\alpha}(P_{Y|X}, \hat{P}_{Y|X}) = \frac{\alpha}{1-\alpha} \frac{1}{N} \sum_{x,y} P_{Y|X}(y|x) \hat{P}_{Y|X}^\frac{\alpha-1}{\alpha}(y|x)$ for $\alpha \in (0,\infty)$. It is obviously based on the Arimoto divergence like~\cite{Welfert2024,Veiner2024AlphaGAN,Kurri2022ISIT_AlphaGAN,Kurri2021ITW_AlphaGAN}, where the loss at $\alpha = 0$ is not well defined. 
The results in~\cite{Sypherd2019,Sypherd2022} show that tuning $\alpha$ can help deal with problems such as label flips and imbalanced classes. They are also the expected performance improvements, if we formulate a more formal weighted power mean loss. 

\section{Conclusion}

We proposed an $\alpha$-GAN model where the value function is a R\'{e}nyi cross entropy. We showed that when $\alpha = 1$, $1$-GAN is equivalent to vanilla GAN, where the value function is exactly the binary cross entropy. 
We derive the optimal saddle point solution for the proposed $\alpha$-GAN and revealed that the exponentially increasing gradient when $\alpha \to 0$. We ran two experiments and showed that convergence is faster when $\alpha \in (0,1)$.

This paper initiates the study on the role of $\alpha$ in GAN training. While the analysis shows accelerated gradient in small values of $\alpha$, the study can be deepened to see whether other problems, e.g., model collapse, can be treated by varying $\alpha$. 
The other problem concerns the concavity of R\'{e}nyi entropy in $\min_{P_g}  V_\alpha (D^*,P_g) = - \max_{P_g} H_{\alpha} (P_{Z|X})$ for training the generator. As $H_{\alpha} (P_{Z|X})$ is only concave when $\alpha \geq 1$, the convergence for $\alpha < 1$ needs to be explored. Another choice is to adopt R\'{e}nyi divergence, which is convex in the second argument for the entire $\alpha \in [0,\infty)$~\cite[Theorem~12]{Erven2014_JOURNAL}.

{
	\small
	\bibliographystyle{nips}
	\bibliography{BIB1}
}


%
%
%
%
%


\newpage

\appendix

\section{Technical Appendices and Supplementary Material}

For experiment settings described in Section~\ref{sec:Exp}, Figures~\ref{fig:SupHist0.1}, \ref{fig:SupHist0.5}, \ref{fig:SupHist1} and~\ref{fig:SupHist5} show the convergence performance when our proposed $\alpha$-GAN is trained for learning $1$D Gaussian distribution for $\alpha=0.1$, $\alpha=0.5$, $\alpha=1$ and $\alpha=5$, respectively. The figures display real and generated distributions for every $1000$ epochs, which complement Figure~\ref{fig:1DGaussian}. 

\begin{figure}[h]
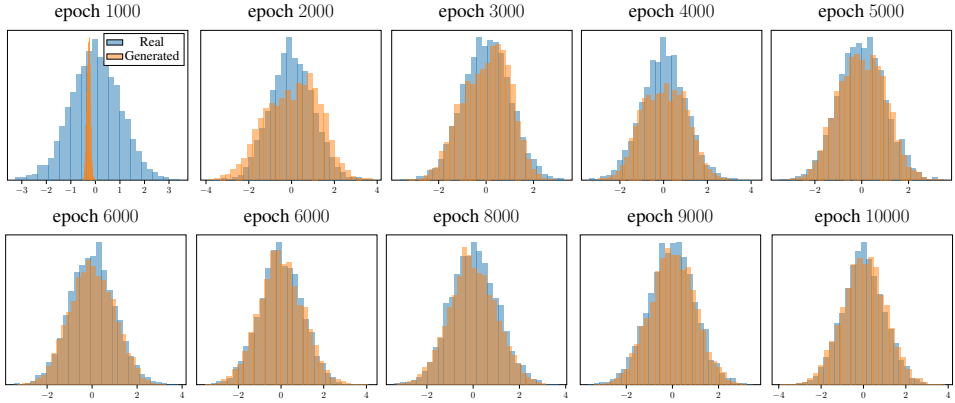

	\begin{tabular}[b]{c}
	\scalebox{0.35}{\input{figures/1DGaussHist_0.1_1000.tex}}
	\scalebox{0.35}{\input{figures/1DGaussHist_0.1_2000Sup.tex}}
	\scalebox{0.35}{\input{figures/1DGaussHist_0.1_3000.tex}}
	\scalebox{0.35}{\input{figures/1DGaussHist_0.1_4000.tex}} 
	\scalebox{0.35}{\input{figures/1DGaussHist_0.1_5000.tex}}\\
	\scalebox{0.35}{\input{figures/1DGaussHist_0.1_6000.tex}}
	\scalebox{0.35}{\input{figures/1DGaussHist_0.1_7000.tex}}
	\scalebox{0.35}{\input{figures/1DGaussHist_0.1_8000.tex}} 
	\scalebox{0.35}{\input{figures/1DGaussHist_0.1_9000.tex}}
	\scalebox{0.35}{\input{figures/1DGaussHist_0.1_10000.tex}}
	\end{tabular}
	\caption{Training $\alpha$-GAN for learning $1$D Gaussian for $\alpha=0.1$}
	\label{fig:SupHist0.1}
\end{figure}

\begin{figure}[h]
	\begin{tabular}[b]{c}
	\scalebox{0.35}{\input{figures/1DGaussHist_0.5_1000.tex}}
	\scalebox{0.35}{\input{figures/1DGaussHist_0.5_2000Sup.tex}}
	\scalebox{0.35}{\input{figures/1DGaussHist_0.5_3000.tex}}
	\scalebox{0.35}{\input{figures/1DGaussHist_0.5_4000.tex}} 
	\scalebox{0.35}{\input{figures/1DGaussHist_0.5_5000.tex}}\\
	\scalebox{0.35}{\input{figures/1DGaussHist_0.5_6000.tex}}
	\scalebox{0.35}{\input{figures/1DGaussHist_0.5_7000.tex}}
	\scalebox{0.35}{\input{figures/1DGaussHist_0.5_8000.tex}} 
	\scalebox{0.35}{\input{figures/1DGaussHist_0.5_9000.tex}}
	\scalebox{0.35}{\input{figures/1DGaussHist_0.5_10000.tex}}
	\end{tabular}
	\caption{Training $\alpha$-GAN for learning $1$D Gaussian for $\alpha=0.5$}
	\label{fig:SupHist0.5}
\end{figure}

\begin{figure}[h]
	\begin{tabular}[b]{c}
	\scalebox{0.35}{\input{figures/1DGaussHist_1_1000.tex}}
	\scalebox{0.35}{\input{figures/1DGaussHist_1_2000Sup.tex}}
	\scalebox{0.35}{\input{figures/1DGaussHist_1_3000.tex}}
	\scalebox{0.35}{\input{figures/1DGaussHist_1_4000.tex}} 
	\scalebox{0.35}{\input{figures/1DGaussHist_1_5000.tex}}\\
	\scalebox{0.35}{\input{figures/1DGaussHist_1_6000.tex}}
	\scalebox{0.35}{\input{figures/1DGaussHist_1_7000.tex}}
	\scalebox{0.35}{\input{figures/1DGaussHist_1_8000.tex}} 
	\scalebox{0.35}{\input{figures/1DGaussHist_1_9000.tex}}
	\scalebox{0.35}{\input{figures/1DGaussHist_1_10000.tex}}
	\end{tabular}
	\caption{Training $\alpha$-GAN for learning $1$D Gaussian for $\alpha=1$}
	\label{fig:SupHist1}
\end{figure}

\begin{figure}[h]
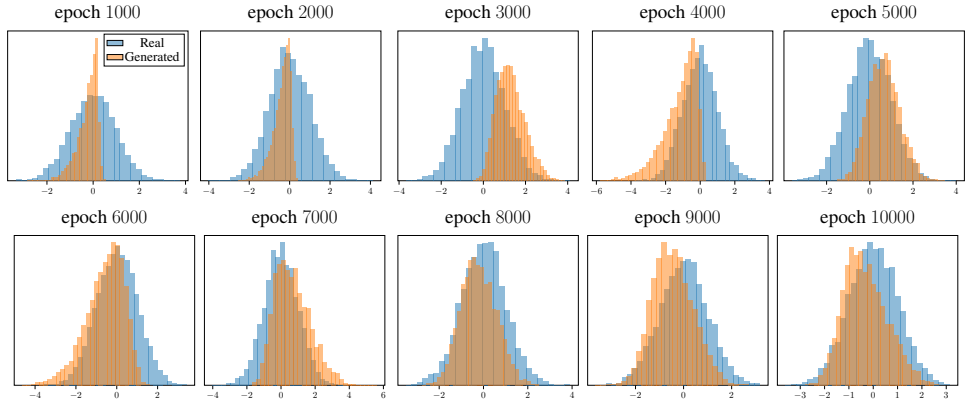

	\begin{tabular}[b]{c}
	\scalebox{0.35}{\input{figures/1DGaussHist_5_1000.tex}}
	\scalebox{0.35}{\input{figures/1DGaussHist_5_2000Sup.tex}}
	\scalebox{0.35}{\input{figures/1DGaussHist_5_3000.tex}}
	\scalebox{0.35}{\input{figures/1DGaussHist_5_4000.tex}} 
	\scalebox{0.35}{\input{figures/1DGaussHist_5_5000.tex}}\\
	\scalebox{0.35}{\input{figures/1DGaussHist_5_6000.tex}}
	\scalebox{0.35}{\input{figures/1DGaussHist_5_7000.tex}}
	\scalebox{0.35}{\input{figures/1DGaussHist_5_8000.tex}} 
	\scalebox{0.35}{\input{figures/1DGaussHist_5_9000.tex}}
	\scalebox{0.35}{\input{figures/1DGaussHist_5_10000.tex}}
	\end{tabular}
	\caption{Training $\alpha$-GAN for learning $1$D Gaussian for $\alpha=5$}
	\label{fig:SupHist5}
\end{figure}

Figure~\ref{fig:SupD0.1}, \ref{fig:SupD0.5}, \ref{fig:SupD1} and~\ref{fig:SupD5} show the convergence of discriminator's soft decision $D(x)$ for every $1000$ epochs for $\alpha=0.1$, $\alpha=0.5$, $\alpha=1$ and $\alpha=5$, respectively. They complement Figure~\ref{fig:1DGaussianAna}. It can be seen that when $\alpha = 0.1$, $D$ can reach uniform distribution $D(x) = 0.5$ for all $x$ very quickly. 

\begin{figure}[h]
	\begin{tabular}[b]{c}
		\scalebox{0.33}{\input{figures/1DGaussD_0.1_1000.tex}}
		\scalebox{0.33}{\input{figures/1DGaussD_0.1_2000.tex}}
		\scalebox{0.33}{\input{figures/1DGaussD_0.1_3000.tex}}
		\scalebox{0.33}{\input{figures/1DGaussD_0.1_4000.tex}} 
		\scalebox{0.33}{\input{figures/1DGaussD_0.1_5000.tex}}\\
		\scalebox{0.33}{\input{figures/1DGaussD_0.1_6000.tex}}
		\scalebox{0.33}{\input{figures/1DGaussD_0.1_7000.tex}}
		\scalebox{0.33}{\input{figures/1DGaussD_0.1_8000.tex}} 
		\scalebox{0.33}{\input{figures/1DGaussD_0.1_9000.tex}}
		\scalebox{0.33}{\input{figures/1DGaussD_0.1_10000Sup.tex}}
	\end{tabular}
	\caption{The change discriminator's soft decision $D(x)$ in Figure~\ref{fig:SupHist0.1} for $\alpha=0.1$.}
	\label{fig:SupD0.1}
\end{figure}

\begin{figure}[h]
	\begin{tabular}[b]{c}
		\scalebox{0.33}{\input{figures/1DGaussD_0.5_1000.tex}}
		\scalebox{0.33}{\input{figures/1DGaussD_0.5_2000.tex}}
		\scalebox{0.33}{\input{figures/1DGaussD_0.5_3000.tex}}
		\scalebox{0.33}{\input{figures/1DGaussD_0.5_4000.tex}} 
		\scalebox{0.33}{\input{figures/1DGaussD_0.5_5000.tex}}\\
		\scalebox{0.33}{\input{figures/1DGaussD_0.5_6000.tex}}
		\scalebox{0.33}{\input{figures/1DGaussD_0.5_7000.tex}}
		\scalebox{0.33}{\input{figures/1DGaussD_0.5_8000.tex}} 
		\scalebox{0.33}{\input{figures/1DGaussD_0.5_9000.tex}}
		\scalebox{0.33}{\input{figures/1DGaussD_0.5_10000Sup.tex}}
	\end{tabular}
	\caption{The change discriminator's soft decision $D(x)$ in Figure~\ref{fig:SupHist0.5} for $\alpha=0.5$.}
	\label{fig:SupD0.5}
\end{figure}

\begin{figure}[h]
	\begin{tabular}[b]{c}
		\scalebox{0.33}{\input{figures/1DGaussD_1_1000.tex}}
		\scalebox{0.33}{\input{figures/1DGaussD_1_2000.tex}}
		\scalebox{0.33}{\input{figures/1DGaussD_1_3000.tex}}
		\scalebox{0.33}{\input{figures/1DGaussD_1_4000.tex}} 
		\scalebox{0.33}{\input{figures/1DGaussD_1_5000.tex}}\\
		\scalebox{0.33}{\input{figures/1DGaussD_1_6000.tex}}
		\scalebox{0.33}{\input{figures/1DGaussD_1_7000.tex}}
		\scalebox{0.33}{\input{figures/1DGaussD_1_8000.tex}} 
		\scalebox{0.33}{\input{figures/1DGaussD_1_9000.tex}}
		\scalebox{0.33}{\input{figures/1DGaussD_1_10000Sup.tex}}
	\end{tabular}
	\caption{The change discriminator's soft decision $D(x)$ in Figure~\ref{fig:SupHist1} for $\alpha=1$.}
	\label{fig:SupD1}
\end{figure}

\begin{figure}[h]
	\begin{tabular}[b]{c}
		\scalebox{0.33}{\input{figures/1DGaussD_5_1000.tex}}
		\scalebox{0.33}{\input{figures/1DGaussD_5_2000.tex}}
		\scalebox{0.33}{\input{figures/1DGaussD_5_3000.tex}}
		\scalebox{0.33}{\input{figures/1DGaussD_5_4000.tex}} 
		\scalebox{0.33}{\input{figures/1DGaussD_5_5000.tex}}\\
		\scalebox{0.33}{\input{figures/1DGaussD_5_6000.tex}}
		\scalebox{0.33}{\input{figures/1DGaussD_5_7000.tex}}
		\scalebox{0.33}{\input{figures/1DGaussD_5_8000.tex}} 
		\scalebox{0.33}{\input{figures/1DGaussD_5_9000.tex}}
		\scalebox{0.33}{\input{figures/1DGaussD_5_10000Sup.tex}}
	\end{tabular}
	\caption{The change discriminator's soft decision $D(x)$ in Figure~\ref{fig:SupHist5} for $\alpha=5$.}
	\label{fig:SupD5}
\end{figure}

In Table~\ref{tab:MNIST}, we list the FID scores correspond to the experimental results in Figure~\ref{fig:MNIST}. 

\begin{table}[H]
  \caption{FID score on MNIST dataset}
  \label{tab:MNIST}
  \centering
  \begin{tabular}{ll}
	    \toprule
	    $\alpha$    & FID score  \\
	    \midrule
	    $0.1$ & $190.48$     \\
	    $0.5$ & $213.06$     \\
	    $1$   & $221.10$  \\
	    $5$   & $241.83$  \\
	    \bottomrule
	  \end{tabular}
\end{table}

\end{document}